\documentclass[11pt]{article}
\usepackage{amsfonts}
\usepackage{amsmath}
\usepackage{amssymb}
\usepackage{graphicx}
\usepackage{algorithm}
\usepackage{algorithmic}
\usepackage{cases}
\usepackage{epstopdf}
\usepackage{geometry, tabularx}
\usepackage{multirow,makecell}
\usepackage{enumerate}
\usepackage{bm}
\usepackage{pifont}
\usepackage{footnote}
\usepackage{titlesec}
\usepackage[titletoc]{appendix}
\usepackage{setspace}
\usepackage[sectionbib]{natbib}%\usepackage[numbers]{natbib}
\usepackage{mathrsfs}
\usepackage{threeparttable}

\providecommand{\U}[1]{\protect\rule{.1in}{.1in}}

\newtheorem{thm}{\bf Theorem}      % Defines the environments for Theorems,
\newtheorem{cor}{\bf Corollary}[section]     % Corollaries, Lemmas, etc.

\newtheorem{prop}{\bf Proposition}

\newtheorem{rem}{\bf Remark}

\usepackage{hyperref}
\hypersetup{hypertex=true,
            colorlinks=true,
            linkcolor=blue,
            anchorcolor=green,
            citecolor=blue}

\usepackage{graphicx,color}
\usepackage[bf,SL,BF]{subfigure}

\newenvironment{proof}[1][Proof]{\noindent\textbf{#1.} }{\ \rule{0.5em}{0.5em}}
\normalsize
\titlespacing*{\section} {0pt}{9pt}{0pt}
\numberwithin{equation}{section}
\allowdisplaybreaks

\usepackage{booktabs}
\usepackage{diagbox}
\begin{document}
\normalsize
\title{A  unified consensus-based parallel ADMM algorithm for high-dimensional  regression with combined regularizations}

\author{ Xiaofei Wu\thanks{College of Mathematics and Statistics, Chongqing University, Chongqing, 401331, P.R. China. Email: xfwu1016@163.com}, \ \  Zhimin Zhang\thanks{Corresponding author. College of Mathematics and Statistics, Chongqing University, Chongqing, 401331, P.R. China. Email: zmzhang@cqu.edu.cn}, \ \ Zhenyu Cui \thanks{School of Business, Stevens Institute of Technology, Hoboken, United States, NJ 07030. Email: zcui6@stevens.edu}.}
\date{}
\maketitle
\vspace{-.40in}

\begin{abstract}
The parallel alternating direction method of multipliers (ADMM) algorithm is widely recognized for its effectiveness in handling large-scale datasets stored in a distributed manner, making it a popular choice for solving statistical learning models. However, there is currently limited research on parallel algorithms specifically designed for high-dimensional  regression with combined (composite) regularization terms. These terms, such as elastic-net, sparse group lasso, sparse fused lasso, and their nonconvex variants, have gained significant attention in various fields due to their ability to incorporate prior information and promote sparsity within specific groups or fused variables. The scarcity of parallel algorithms for combined regularizations can be attributed to the inherent nonsmoothness and complexity of these terms, as well as the absence of closed-form solutions for certain proximal operators associated with them.  In this paper, we propose a  \textit{unified}  constrained optimization formulation based on the consensus problem for these types of convex and nonconvex regression problems and derive the corresponding parallel ADMM algorithms. Furthermore, we prove that the proposed algorithm not only has global convergence but also exhibits linear convergence rate. Extensive simulation experiments, along with a financial example, serve to demonstrate the reliability, stability, and scalability of our algorithm. The R package  for implementing the proposed algorithms can be obtained at \url{https://github.com/xfwu1016/CPADMM}.
\end{abstract}

\section{Introduction}
The advancement of modern science and technology has made data collection increasingly easy, resulting in the generation of massive amounts of data. However, due to the sheer volume of data and other factors such as privacy concerns, it has become necessary to store this data in a distributed manner. Therefore, it is essential to design parallel algorithms that can adapt to this massive and distributed storage dataset. Interested readers can refer to \cite{SNEBJ}, \cite{P2014}, \cite{LFL}, and their references for further information.

In this paper, our focus is on high-dimensional regression problems with massive data. When not considering distributed storage, the data required for a regression model typically includes the following,
\begin{equation}\label{data}
\{ {y_i},{\boldsymbol{x}_i}\} _i^n = \{ {y_i},{x_{i,1}},{x_{i,2}}, \cdots ,{x_{i,p}}\} _i^n =\{\boldsymbol y, \boldsymbol{X} \}.
\end{equation}
Here, $y_i$ denotes the target value for the $i$-th observation $\boldsymbol{x}_i$, and $\boldsymbol{x}_i$ represents a $p$-dimensional vector that signifies the $i$-th observation value of $p$ features. These features could be the original observations and/or selected functions constructed from them.
The objective of regression is to establish a relationship between the features in $\bm X$ and the target values in $\bm y$ by estimating coefficients that minimize the difference between the predicted and actual values of $\bm y$.  To avoid overfitting and improve interpretability, some regularization terms    are considered to be added to the objective optimization function, that is,
\begin{align}\label{hdlr}
\mathop {\arg \min }\limits_{\bm \beta \in \mathbb{R}^p} \quad \mathcal{L} (\bm y - \bm X \bm \beta)  + P_\lambda(|\bm \beta|),
\end{align}
where  $\mathcal{L} (\bm y - \bm X \bm \beta) = \frac{1}{n}\sum_{i=1}^{n} {L} (y_i - \bm x_i^\top \bm \beta)$.  Here, the loss function ${L}$ is a generic function that includes least squares,  Huber loss in \cite{H1964} and quantile loss in \cite{K1970}.  We also consider situations where $\mathcal{L} (\bm y - \bm X \bm \beta)$ cannot be separated, such as square root in \cite{BCW}. The regularization term $P_\lambda(|\bm \beta|)$ with tuning parameter $\lambda > 0$  can be a single regularization term, such as $\ell_2$ or ridge in \cite{H1970}, total variation in \cite{RO}, lasso in (\cite{T}), group lasso in \cite{YL2008}; or a combination of them, such as sparse fused lasso (lasso plus total variation) in \cite{TSRZ}, elastic-net (lasso plus $\ell_2$) in \cite{Z},  sparse group lasso (lasso plus group lasso) in \cite{W2008}; or their nonconvex variants.  Due to the lack of unified algorithms for these combined regularization  regressions, this paper focuses primarily on addressing this issue.

The alternating direction method of multipliers (ADMM)  is an iterative optimization algorithm commonly used to solve convex minimization problems with linear constraint. 
It decomposes these problems into subproblems and can be applied to various statistical learning models. \cite{SNEBJ} provided  a good overview of its development and  detailed applications in various fields.
Because some of the functions $L$ are nonsmooth and nondifferentiable, it is necessary to introduce the linear constraint $\bm{r}=\bm{y}-\bm{X}\bm{\beta}$ in order to use ADMM to uniformly solve (\ref{hdlr}). The constrained optimization problem is then formulated as follows,
\begin{align}\label{intr1}
\min_{\bm \beta, \bm r} & \quad   \mathcal{L}(\bm r)  + {P}_\lambda(|\bm \beta|), \notag \\
\text{s.t.} \ & \bm X \bm \beta  + \bm r   = \bm y.
\end{align}
Thus, the augmented Lagrangian form of (\ref{intr1}) is
\begin{align}
L_\mu(\bm \beta, \bm r, \bm d) = 
\mathcal{L}(\bm r)  + {P}_\lambda(|\bm \beta|)  - \bm d^\top (\bm X \bm \beta  + \bm r  - \bm y) +  \frac{\mu}{2}  \|  \bm X \bm \beta + \bm r - \bm y \|_2^2,
\end{align}
where $\bm d$  is dual variable corresponding to the linear constraint, and $\mu>0$ is a given penalty parameter.
The iterative scheme of ADMM for (\ref{intr1}) is
\begin{equation}\label{twoupadmm}
\left\{ \begin{array}{l}
\bm \beta^{k+1} \ \leftarrow  \mathop {\arg \min }\limits_{\bm \beta} \left\{ L_\mu (\bm \beta, \bm r^k,  \bm d^k) \right \};\\
\bm r^{k+1} \ \leftarrow  \mathop {\arg \min }\limits_{\bm r} \left\{ L_\mu (\bm \beta^{k+1}, \bm r, \bm d^k)\right \}; \\
\bm d^{k+1} \ \leftarrow  \bm{d}^{k} - \mu(\bm X \bm \beta^{k+1} + \bm r^{k+1} - \bm y).
\end{array} \right.
\end{equation}

In general, when designing algorithms for distributed parallel processing, it requires a central machine and several  local machines. Assume that 
\begin{align}\label{mdata}
\bm X =(\bm X_1^\top, \bm X_2^\top, \dots, \bm X_M^\top)^\top \quad \text{and} \quad  \bm y = (\bm y_1^\top, \bm y_2^\top, \dots, \bm y_M^\top)^\top
\end{align}
are stored in a distributed manner across $M$ local machines. The central machine receives information transmitted by the local machines, consolidates the information, and then forwards it to the respective local machines. The decomposition in (\ref{mdata}) enables the algorithm to be naturally parallelizable across multiple local machines. Each local machine can independently solve its designated subproblem, and the solutions are then communicated and coordinated between the central machine and the local machines through variable updates. This distributed storage allows for parallel processing and efficient handling of large-scale datasets.
The traditional ADMM algorithm in (\ref{twoupadmm}) can only handle problems on a single machine and is not sufficient for data stored in a distributed manner.
Recently, many parallel ADMM algorithms have been proposed to solve statistical learning models when data is stored in a distributed manner, such as linear regression models (see \cite{SNEBJ}, \cite{YLW} and  \cite{FLY}).  The existing parallel  ADMM  algorithms employ strategies similar to the regularized consensus problem described by \cite{SNEBJ}. 
The consensus problem in distributed computing refers to the need for multiple nodes in a distributed system to agree on a common decision.

These existing parallel algorithms are specifically designed for individual regularization terms which can  achieve sparse feature selection, such as  lasso, group lasso, scad (\cite{FL}) and mcp (\cite{ZC}). However, practical applications often involve other types of structural information beyond sparse feature structures. These include high correlation, feature grouping, or smooth trends among features. Combined regularization terms can simultaneously utilize different regularization techniques to integrate multiple structural information in the model, enhancing its performance and generalization ability.
Thus, in such cases, many studies have suggested replacing a single regularization term with a combined regularization term. For example, elastic-net  in \cite{Z2005}, fused lasso  in \cite{TSRZ}, and sparse group lasso  in \cite{W2008}.
Moreover, many nonconvex variants of these combined regularization terms have been developed. For example, elastic-net variants such as snet (scad plus $\ell_2$) in \cite{W2010} and mnet (mcp plus $\ell_2$) in \cite{HBLMZ}; sparse fused lasso variants like sctv (\textbf{sc}ad plus \textbf{t}otal \textbf{v}ariation) and mctv (\textbf{mc}p plus \textbf{t}otal \textbf{v}ariation) in \cite{XLLK}; sparse group lasso variants such as scgl (\textbf{sc}ad plus \textbf{g}roup \textbf{l}asso) and mcgl (\textbf{mc}p plus \textbf{g}roup \textbf{l}asso) in \cite{H2012} and \cite{T2021b}.

However, when the individual regularization  is  replaced by the  combined regularization,  the parallel algorithms mentioned above cannot be directly applied to these combined regularization regressions due to the complexity of the combination penalty terms, as well as the non-separability of certain composite regularization terms (such as sparse fused lasso and its nonconvex variants), which do not have closed-form solutions for their proximal operators. Currently, there are no parallel algorithms proposed for solving these high-dimensional  problems with combined regularizations. 
The main work of this paper is to fill this gap and develop a  unified consensus-based parallel ADMM (CPADMM) algorithm for high-dimensional regression problems with combined regularizations. The proposed algorithm has three main advantages: 1. It can not only handle massive amounts of data, but also adapt to the structure of distributed data storage. 2. Our algorithm is applicable to a wide range of convex loss functions, as long as we can derive the closed-form solutions for the proximal operators   corresponding to the losses. 3. It has a guarantee of global convergence and a linear convergence rate.

The rest of the paper is organized as follows. Section \ref{sec2}  introduces some widely used combined regularization terms, and introduces some proximal  operators that can play a crucial role in the algorithm. In Section \ref{sec3}, we first review the existing parallel ADMM algorithms for high-dimensional linear regression, and then propose a unified regression parallel ADMM algorithm for solving combined regularization. Section \ref{sec4} discusses the specific implementation of the algorithm for solving various  combined regularization regression. The convergence analysis is provided in Section \ref{sec5}.  Section \ref{sec6} presents numerical results demonstrating the scalability, efficiency and accuracy of the proposed algorithm. Section \ref{sec7} summarizes the findings and concludes the paper, with a discussion on future research directions. Technical proofs, algorithm extensions and additional experiments are provided in the Appendix. The R package CPADMM for implementing the parallel ADMM algorithms can be obtained at \url{https://github.com/xfwu1016/CPADMM}.

\textbf{Notations}:
$\bm 0_n$ and $\bm 1_n$ represent $n$-dimensional vectors with all elements being 0 and 1, respectively.
$\bm F$ is a $(p-1) \times p$ matrix with all elements being 0, except for 1 on the diagonal and -1 on the superdiagonal.
$\bm I_n$ represents the $n$-dimensional identity matrix.
The Hadamard product is denoted by $\odot$.
The sign$(\cdot)$ function is defined component-wise such that sign$(t) = 1$ if $t > 0$, sign$(t) = 0$ if $t = 0$, and sign$(t) = -1$ if $t < 0$.
$(\cdot)_+$ signifies the element-wise operation of extracting the positive component, while $|\cdot|$ denotes the element-wise absolute value function.
For any vector $\bm u$, $\|\bm u\|_1$, $\|\bm u\|_2$, and $\|\bm u\|_{2,1}$ denote the $\ell_1$ norm, the $\ell_2$ norm, and the $\ell_{2,1}$ norm of $\bm u$, respectively.
$\|\bm u\|_{\bm H} := \sqrt{\bm u^\top \bm H \bm u}$ is used to denote the norm of $\bm u$ under the matrix $\bm H$, where $\bm H$ is a matrix.

\newpage

\section{Preliminaries and Literature Review}\label{sec2}
In this section, we first review some commonly used combined regularization techniques and loss functions. Then we derive their corresponding proximal operators, as they play a crucial role in our algorithm.
\vspace{-1em}
\subsection{Combined regularizations}
The combined regularization can be defined as 
\begin{align}
P_\lambda(|\bm \beta|) =  P_{\lambda_1}(|\bm \beta|) +  P_{\lambda_2}(|\bm \beta|),
\end{align}
where $\lambda_1, \lambda_2 > 0$ are two  tuning parameters, and $P_{\lambda_1}(|\bm \beta|)$ and $P_{\lambda_2}(|\bm \beta|)$ are two individual regularization terms. Next, we will discuss in detail the various expressions for combined regularization.
\vspace{-1em}
\subsubsection{Elastic-net}
\vspace{-1em}
Elastic-net regularization in \cite{Z2005} is a popular regularization technique used in machine learning and statistics. It is specifically designed for scenarios where there are a large number of variables and multicollinearity exists among them. This method combines the strengths of two other regularization techniques, $\ell_1$  (lasso) and $\ell_2$ (ridge), that is
\begin{align}
\lambda_1\| \bm \beta \|_1 + \lambda_2 \|\bm \beta \|_2^2,
\end{align}
to achieve a more effective and stable model. Due to these advantages, elastic-net regularization is widely applied in various fields, including  high-dimensional data analysis in \cite{ZT2006}, \cite{ZZ2009} and \cite{ZX2018};  genomics and bioinformatics in \cite{O2012}, \cite{ZH2015} and \cite{Z2019}; and feature selection and variable screening in  \cite{Z2020} and \cite{FJQ2020}.
Many algorithms have also been proposed to solve various loss functions with elastic-net regularization, such as path solution algorithm in  \cite{RZ}; cyclical coordinate descent in \cite{JTHR, FHT},  \cite{YZ2013} and \cite{Y2017}; ADMM algorithm in \cite{GFK} and \cite{L2023}. Although these algorithms perform well in high-dimensional data, they are not suitable for handling massive data stored in distributed storage.
\vspace{-1em}
\subsubsection{Sparse fused lasso}
\vspace{-1em}
For problems in which covariates are sparse and blocky  structure are desired, the sparse fused lasso   regularization (\cite{TSRZ}) has proved to be very efficient.  The  sparse fused lasso regularization term is defined as 
\begin{align} 
{\lambda}_1  {\left \| {{\boldsymbol{\beta}}} \right \|_1}  + \lambda_2 \sum\limits_{j = 2}^p \left| {{\boldsymbol{\beta} _j} - {\boldsymbol{\beta} _{j - 1}}} \right| =  {\lambda}_1 {\left \| {{\boldsymbol{\beta}}} \right \|_1}  + \lambda_2 \left \| \bm F \bm \beta \right \|_1
\end{align}
We refer to \cite{LZ}, \cite{TP},  \cite{AE}, \cite{YHMF}, \cite{PS2016} and \cite{CSM} for wide applications of the fused lasso regularization. 
There have been several algorithms proposed to solve various types of high-dimensional  fused lasso regression ($\bm X$  is not an identity matrix), including the majorization-minimization algorithm (\cite{YWL}),  various variations of the ADMM  algorithms (\cite{YX}, \cite{LMY}, \cite{JLD} and \cite{W2023}).
However, the above algorithms are also not suitable for massively distributed storage of data.
\vspace{-1em}
\subsubsection{Sparse group lasso}
\vspace{-1em}
 In high-dimensional supervised learning problems, leveraging problem-specific assumptions often leads to improved accuracy. 
For problems involving grouped covariates, it is widely believed that the effects are sparsely distributed at both the group level and within the groups. Many studies (see \cite{F2010}, \cite{Q2012}, \cite{S2013} and \cite{CH2020}) have shown that employing sparse group lasso regularization  can accurately identify these grouped covariates. The  sparse group lasso regularization  is defined as 
\begin{align}
 {\lambda}_1   {\left \| {{\boldsymbol{\beta}}} \right \|_1}  + \lambda_2 {\left\| \bm \beta \right\|_{2,1}},
\end{align}
where $\ell_{2,1}$-norm  is a rotational invariant of the $\ell_1$ norm.  If a $p$-vector $\bm \beta$ is partitioned into $G$ disjoint groups denoted respectively by $\bm \beta_1, \bm \beta_2, \dots, \bm \beta_G$,  then its  $\ell_{2,1}$-norm is defined as
\begin{align}
{\left\| \bm \beta \right\|_{2,1}} = \sum\limits_{g=1}^{G} \left\| \bm \beta \right\|_{2}
\end{align}
We refer to \cite{H2009}, \cite{Z2015}, \cite{C2020}, \cite{L2020} and \cite{H2022} for wide applications of the group lasso regularization.  So far, several algorithms have been proposed to solve the high-dimensional sparse group regularization regression problem,  including but not limited to  coordinate descent algorithm (\cite{W2008}, \cite{D2017} and \cite{J2019}), accelerated generalized gradient descent algorithm (\cite{S2013}), linearized ADMM algorithm (\cite{LMY}), inexact semismooth newton based augmented lagrangian algorithm (\cite{Zh2020}), and  proximal gradient descent algorithm (\cite{K2020}). However, these algorithms are not suitable for massive data stored in distributed storage.
\vspace{-1em}
\subsubsection{Nonconvex extensions for combined  regularizations} \label{sec214}
\vspace{-1em}
Although these combined  regularizations effectively utilize prior information about the structure of features, they still rely on lasso to achieve sparsity. Many studies, such as \cite{FL} and \cite{Z}, have confirmed that while lasso has the ability to induce sparsity, it can introduce non-negligible estimation bias due to the use of the same penalty for all terms. Therefore,  Many non-convex penalty functions have been proposed for solving this problem, and the most famous ones among them are scad (\cite{FL}) and mcp (\cite{ZC}).  In this paper, we mainly focus our attention on two popular nonconvex regularizations, $\text{scad}(|\bm \beta|)$ and $\text{mcp}(|\bm \beta|)$, to replace $\lambda_1\|\bm\beta \|_1$ to form nonconvex combined regularization. We summarize these combined regularizations in Table \ref{Tab1}.  These nine combined regularizations are also the main focus of this paper.
\begin{table}[h]\small
\centering
\caption{\footnotesize{Nine types of combined regularization.}}
\renewcommand{\arraystretch}{1.5}
\begin{tabular}{ccccc}
\Xhline{1pt}
                        & $P_{\lambda_2}(|\bm \beta|) = \lambda_2\|\bm\beta \|_2^2$ & $P_{\lambda_2}(|\bm \beta|)=\lambda_2 {\left\| \bm \beta \right\|_{2,1}}$  & $P_{\lambda_2}(|\bm \beta|) = \lambda_2 \left \| \bm F \bm \beta \right \|_1$       \\
\Xhline{0.7pt}
$P_{\lambda_1}(|\bm \beta|) = \lambda_1\|\bm\beta \|_1$ & elastic-net & sgla   & sfla   \\
$P_{\lambda_1}(|\bm \beta|) = \text{scad}(|\bm\beta |)$ & snet        & scgl & sctv \\
$P_{\lambda_1}(|\bm \beta|) = \text{mcp}(|\bm\beta |)$ & mnet        & mcgl  & mctv  \\
\Xhline{1pt}
\end{tabular}
\label{Tab1}
\end{table}

Local linear approximation (LLA) method proposed by \cite{ZL} is an effective algorithm for solving nonconvex regularization regression. The main idea is to transform the statistical model algorithm for solving scad and mcp into solving several iterative reweighted lasso problems. \cite{FLY2014} gave this method a theoretical guarantee, proving that for many losses,  as long as the initial
estimator and the oracle estimator behave well, the two step LLA can find nice estimator with large probability. In this paper, we will use LLA method to transform these nonconvex combined regularization optimization forms into weighted convex optimization forms.

\subsection{Proximal operator}
Proximal operator introduced by \cite{P2014} is widely used in convex optimization and numerical optimization, particularly when the objective function involves nondifferentiable and/or nonsmooth terms. Its mathematical expression is as follows,
\begin{align}
\operatorname{prox}_{\gamma, f} (\bm x) = \arg \min_{\bm z \in \mathbb{R}^q} \left( f(\bm z) + \frac{\gamma}{2} \| \bm z - \bm x \|_2^2 \right),
\end{align}
where $f : \mathbb{R}^q \rightarrow \mathbb{R} \cup \{+\infty\}$ is a closed proper convex function, $\bm z$ and $\gamma$  respectively denote a given vector and a constant. Generally speaking, $f$ is assumed to be separable, that is $f(\bm z) = \sum_{i} f_i(z_i) $.
Since $\| \cdot \|_2^2$ is also a separable function, $\operatorname{prox}_{\gamma, f} (\bm x)$ can be solved coordinate-wise, that is
\begin{align}
\operatorname{prox}_{\gamma, f_i} (x_i) = \arg \min_{z_i} \left(f_i(z_i) +\frac{\gamma}{2} (z_i - x_i)^2  \right).
\end{align}
\subsubsection{Proximal operators for  regularizations} 
Many widely used operators in applications are actually special cases of proximal operators, among which the most famous one is the soft-thresholding operator in \cite{D1995}. The soft-thresholding operator is defined as
$\text{prox}_{\gamma,  \lambda\| \bm \beta \|_{1}} (\bm x) = \arg \min_{\bm \beta \in \mathbb{R}^p} \left( \lambda\| \bm \beta \|_{1} + \frac{\gamma}{2} \| \bm \beta - \bm x \|_2^2 \right).$
The closed-form solution for $\text{prox}_{\gamma,  \lambda\| \bm \beta \|_{1}} (\bm x)$ is given by
\begin{align}\label{lasso}
\text{prox}_{\gamma,  \lambda\| \bm \beta \|_{1}} (\bm x) = \text{sign}(\bm x) \odot \max(|\bm x| - \frac{\lambda}{\gamma}, 0).
\end{align}
In addition, we need to consider the proximal operator commonly used for $\ell_{2,1}$-norm regularization, which is also known as the group soft-thresholding operator. It is defined as $\text{prox}_{\gamma,\lambda \| \bm \beta \|_{2,1}}(\bm{x}) = \arg \min_{\bm \beta \in \mathbb{R}^p} \left( \lambda\| \bm \beta \|_{2,1}\right.$  $\left. + \frac{\gamma}{2} \| \bm \beta - \bm x \|_2^2 \right)$.
The closed-form solution for $\text{prox}_{\gamma,\lambda \| \bm \beta \|_{2,1}}(\bm{x})$ is given by
\begin{align}\label{glasso}
\text{prox}_{\gamma,  \lambda\| \bm \beta \|_{2,1}} (\bm x) = \frac{\bm x} {\|\bm x \|_2} \cdot \max(\|\bm x \|_2- \frac{\lambda}{\gamma} ,0)
\end{align}
However, when it comes to the proximal operators of total variation, the definition is as follows
\begin{align}\label{flasso2}
 \text{prox}_{\gamma,\lambda \| \bm F \bm \beta \|_{1}}(\bm{x}) = \arg \min_{\bm \beta \in \mathbb{R}^p} \left(  \lambda \sum\limits_{j = 2}^p \left| {{\boldsymbol{\beta} _j} - {\boldsymbol{\beta} _{j - 1}}} \right| + \frac{\gamma}{2} \| \bm \beta - \bm x \|_2^2 \right).
\end{align} 
Due to the fact that total variation  $\sum_{j = 2}^p \left| {{\boldsymbol{\beta} _j} - {\boldsymbol{\beta} _{j - 1}}} \right|$ is an inseparable regularization term, it is not possible to derive a closed-form solution for the proximal operator in (\ref{flasso2}). Although it can be solved using iterative numerical methods, incorporating this method into ADMM iterations in (\ref{twoupadmm}) would result in a double loop, leading to significant computational expenses and time consumption.

Note that for combined regularizations, such as elastic-net and sparse group lasso, the closed-form solutions for the proximal operators $\text{prox}_{\gamma, \lambda_1\| \bm \beta \|_{1} + \lambda_2 \| \bm \beta \|_{2,1}}(\bm{x})$ and $\text{prox}_{\gamma, \lambda_1\| \bm \beta \|_{1} + \lambda_2 \| \bm \beta \|_{2,1}}(\bm{x})$ can be found in \cite{GFK} and \cite{C2012}, respectively. However,  note also that the closed-form solutions for the proximal operators $\text{prox}_{\gamma, \lambda_1\| \bm \beta \|_{1} + \lambda_2 \| \bm F \bm \beta \|_{1}}(\bm{x})$  cannot be derived  due to the existence of total variation regularization.

\subsubsection{Proximal operators for  loss functions} 
In this paper, we consider several convex loss functions, such as least squares  (\cite{TSRZ}, \cite{Z2005} and \cite{F2010}), quantile (\cite{K2011},  \cite{GFK}  and \cite{W2023}), square root (\cite{B2013} and  \cite{JLD}), and Huber (\cite{CH2020}), that are applied to various high-dimensional elastic-net, sparse fused  lasso and sparse group lasso regression models.
We use the symbol $\mathcal L$ to represent these loss functions, and their proximal operators are defined as follows
\begin{align}
\operatorname{prox}_{\mu, \mathcal L} (\bm x) = \arg \min_{\bm r} \left( \mathcal L(\bm r) + \frac{\mu}{2} \| \bm r - \bm x \|_2^2 \right).
\end{align}
If $\mathcal L$ is a separable function, $\operatorname{prox}_{\mu, \mathcal L} (\bm x)$ can be also solved coordinate-wise, that is
\begin{align}
\operatorname{prox}_{\mu, \mathcal L} (\bm x_i) = \arg \min_{\bm r_i} \left( \frac{1}{n} L(\bm r_i) + \frac{\mu}{2} \| \bm r_i - \bm x_i \|_2^2 \right).
\end{align}
We  summarize the mathematical expressions of these losses and their corresponding closed-form solutions for the proximal operators in Table \ref{Tab2}.  These closed-form solutions have already been discussed in some papers, and a simple and easily extensible proof was presented in the appendix of \cite{L2023}.
\begin{table}[h]\small
\caption{\footnotesize{The mathematical expressions of various loss functions and their closed-form solutions for the proximal operators.}}
\renewcommand{\arraystretch}{1.8}
\begin{tabular}{m{3cm}<{\centering}m{5cm}<{\centering}m{8cm}<{\centering}}
\Xhline{1pt}
Loss function  & Mathematical expression $\mathcal{L}(\bm r)$ &  The closed from of proximal operator   $\operatorname{prox}_{\mu, \mathcal L} (\bm x_i) $    \\
\Xhline{0.4pt}
Least squares  &  $\frac{1}{2n}\sum\limits_{i=1}^n r_i^2$     
               &  $\frac{n\mu x_i}{1+n\mu}$     \\
Quantile       &  $\frac{1}{n}\sum\limits_{i=1}^n r_i(1-\tau I_{(r_i <0)})$       
               &  $\max\{x_i-\frac{\tau}{n\mu},\min(0,x_i+\frac{1-\tau}{n\mu})\}$  \\
Square root    &  $\sqrt{\frac{\sum\limits_{i=1}^n r_i^2}{2n}}$                                                                  
               &  $\frac{x_i}{\|\bm x\|_2}\cdot \max\{\|\bm x\|_2-\frac{1}{\sqrt{2n}\mu},0\}$    \\
Huber          &  $\frac{1}{n}\sum\limits_{i=1}^n [\frac{r_i^2}{2\delta}I_{(r_i\ge \delta)}+(|r_i|-\frac{\delta}{2})I_{(r_i<\delta)}]$    
               &  $\frac{n\mu \delta x_i}{1+n\mu \delta}+\frac{1}{1+n\mu \delta} \cdot \text{sign}(x_i) \cdot \max\{0,|x_i|-\frac{1}{n\mu}-\delta\}$  \\  
\Xhline{1pt}
\end{tabular}
\label{Tab2}
\end{table}

\section{Parallel ADMM algorithm}\label{sec3}
Assuming that there are a total of $M$ local machines available, the response vector $\bm y \in \mathbb{R}^{n}$ and the sample observation matrix $\bm X \in \mathbb{R}^{n \times p}$ can be divided into $M$ blocks as follows,
\begin{align}\label{datasplit}
\bm y = (\bm y_1^\top, \bm y_2^\top, \dots, \bm y_M^\top)^\top, \ \bm X =(\bm X_1^\top, \bm X_2^\top, \dots, \bm X_M^\top)^\top,
\end{align}\label{glasso}
where $\bm y_m \in \mathbb{R}^{n_m}$,  $\bm{X}_m \in \mathbb{R}^{n_m \times p}$ and $\sum_{m=1}^{M} n_{m} =n$.
In this section, we briefly review the existing consensus parallel ADMM algorithms for solving high-dimensional regression  models. Based on existing algorithms, we propose a unified  parallel ADMM algorithm for high-dimensional regression models with combined regularization.
\subsection{Existing consensus-based parallel ADMM algorithm}
\subsubsection{Parallel ADMM  for lasso}
   In  \cite{SNEBJ}, the consensus ADMM algorithm was first introduced for solving high-dimensional lasso regression in distributed storage settings. By introducing  the consensus constraints $\{ \bm \beta = \bm \beta_m \}_{m=1}^{M}$,
the constrained optimization problem for lasso is as follows
\begin{align}\label{classo}
\min_{\bm \beta, \bm \beta_m}  \quad  & \sum_{m=1}^{M} \frac{1}{2n} \|\bm y_m - \bm X_m \bm \beta_m \|_2^2  + \lambda \|\bm \beta\|_1, \notag \\
\text{s.t.} \ & \bm \beta = \bm \beta_m, \  m =1,2,\dots,M.
\end{align}
 The augmented Lagrangian of (\ref{classo}) is 
\begin{align}\label{ac}
L_\mu (\bm \beta,  \bm \beta_m, \bm d_m)  =   \sum_{m=1}^{M} \frac{1}{2n} \|\bm y_m - \bm X_m \bm \beta_m \|_2^2 +  \lambda \|\bm \beta\|_1  -  \sum\limits_{m=1}^{M} \langle \bm  e_m,  \bm \beta_m - \bm \beta \rangle +  \frac{\mu}{2} \sum\limits_{m=1}^{M}  \| \bm \beta_m - \bm \beta \|_2^2,
\end{align}
where $\bm e_m$ is dual variables corresponding to the linear constraint, and $\mu>0$ is a given penalty parameter.
Similar to the iterative steps of the ADMM algorithm, the parallel ADMM can be written in the following form,
\begin{equation}\label{classoi}
\left\{ \begin{array}{l}
\bm \beta^{k+1} \ \leftarrow  \mathop {\arg \min }\limits_{\bm \beta} \left\{  \frac{\mu}{2} \sum\limits_{m=1}^{M}  \| \bm \beta_m^k - \bm \beta - \bm e^k_m/\mu\|_2^2 + \lambda \|\bm \beta \|_1 \right \};\\
\bm \beta_m^{k+1} \ \leftarrow  \mathop {\arg \min }\limits_{\bm \beta_m} \left\{ \frac{1}{2n} \|\bm y_m - \bm X_m \bm \beta_m \|_2^2  + \sum\limits_{m=1}^{M}  \| \bm \beta_m - \bm \beta^{k+1} - \bm e_m^k/\mu\|_2^2   \right \}, \ m =1,2,\dots, M; \\
\bm{e}^{k+1}_m \ \leftarrow  \bm{e}^{k}_m -\mu(\bm \beta_m^{k+1} - \bm \beta^{k+1}),  \ m =1,2,\dots, M.
\end{array} \right.
\end{equation}
The first step involves the central machine receiving the $\bm \beta_m^{k}$ passed by each local machine, integrating the information to obtain $\bm \beta^{k+1}$, and then passing it on to each local machine. In the second step, the update of $\bm \beta_m^{k+1}$ depends only on the updates of $\bm \beta^{k+1}$ and $\bm e_m^k$. This paves the way for parallel computation of $\bm \beta_m, m = 1, 2, 3, \dots, M$. The third step involves updating the dual variables and loading them onto the respective local machines as well.

Note that the least squares loss in (\ref{classo}) is replaced by quantile loss, and the parallel algorithm provided by  \cite{SNEBJ} is not applicable because quantile loss is piecewise linear and non differentiable. Inspired by the work of  \cite{SNEBJ}, \cite{YLW} proposed a consensus-based parallel ADMM algorithm for convex and nonconvex regularized quantile regression.

\subsubsection{Parallel ADMM for regularized quantile regression}
By introducing  the consensus constraints $\{ \bm \beta = \bm \beta_m \}_{m=1}^{M}$ and $\{ \bm r_m = \bm y_m - \bm X_m \bm \beta_m\}_{m=1}^{M}$, the constrained optimization problem is as follows
\begin{align}\label{PCNPQR}
\min_{\bm \beta, \bm r_m, \bm \beta_m} & \quad   \sum_{m=1}^{M}  \rho(\bm r_m)  + {P}_\lambda(|\bm \beta|), \notag \\
\text{s.t.} \ \bm X_m \bm \beta_m  + \bm r_m &  = \bm y_m, \ \bm \beta = \bm \beta_m, \  m =1,2,\dots,M,
\end{align}
where $\bm r = (\bm r_1^\top, \bm r_2^\top, \dots, \bm r_M^\top)^\top$ and $\rho(\bm r_m) =\frac{1}{n} \sum_i \rho (\bm r_{mi}) $. Here, $\rho( \ )$ is the quantile loss that can be found in Table \ref{Tab1}, and ${P}_\lambda(|\bm \beta|)$ is the lasso, scad or mcp. The augmented Lagrangian of (\ref{PCNPQR}) is 
\begin{align}\label{ac}
L_\mu (\bm \beta, \bm r_m, \bm \beta_m, \bm d_m, \bm e_m) & =  \sum\limits_{m=1}^{M}  \rho{{{\left( \bm r_m  \right)}}}  +  {P}_\lambda(\bm \beta)  -  \sum\limits_{m=1}^{M} \langle \bm  d_m, \bm X_m \bm \beta_m + \bm r_m - \bm y_m \rangle   - \sum\limits_{m=1}^{M} \langle \bm  e_m, \bm \beta_m - \bm \beta \rangle \\ \notag
& + \frac{\mu}{2} \sum\limits_{m=1}^{M} \|  \bm X_m \bm \beta_m + \bm r_m - \bm y_m \|_2^2 
 +  \frac{\mu}{2} \sum\limits_{m=1}^{M}  \| \bm \beta_m - \bm \beta \|_2^2.,
\end{align}
where $\bm d_m$ and $\bm e_m$ are dual variables corresponding to the linear constraints, and $\mu>0$ is a given penalty parameter.
With a given $\{ \bm r_m^{0}, \bm \beta_m^{0}, \bm d_m^{0}, \bm e_m^{0} \}_{m=1}^{M}$,  the iterative scheme of parallel ADMM for (\ref{ac}) is
\begin{equation}\label{madmm}
\left\{ \begin{array}{l}
\bm \beta^{k+1} \ \leftarrow \mathop {\arg \min }\limits_{\bm \beta} \left\{ L_\mu (\bm \beta, \bm r_m^k, \bm \beta_m^k, \bm d_m^k, \bm e_m^k) \right \};\\
\bm r_m^{k+1} \ \leftarrow \mathop {\arg \min }\limits_{\bm r_m} \left\{ L_\mu (\bm \beta^{k+1}, \bm r_m, \bm \beta_m^k, \bm d_m^k, \bm e_m^k)\right \},  \ m =1,2,\dots, M; \\
\bm \beta_m^{k+1} \ \leftarrow \mathop {\arg \min }\limits_{\bm \beta_m} \left\{ L_\mu (\bm \beta^{k+1}, \bm r_m^{k+1}, \bm \beta_m, \bm d_m^k, \bm e_m^k) \right \}, \ m =1,2,\dots, M; \\
\bm d^{k+1}_m \ \leftarrow  \bm{d}^{k}_m - \mu(\bm X_m \bm \beta_m^{k+1} + \bm r_m^{k+1} - \bm y_m),  \ m =1,2,\dots, M; \\
\bm{e}^{k+1}_m \ \leftarrow  \bm{e}^{k}_m -\mu(\bm \beta_m^{k+1} - \bm \beta^{k+1}),  \ m =1,2,\dots, M.
\end{array} \right.
\end{equation}
Both $\bm \beta^{k+1}$ and $\bm r^{k+1}$ updates can be written as proximal operators with closed-form solutions. The update of $\bm \beta_m^{k+1}$ involves solving a system of linear equations. The updates of $\bm d_m^{k+1}$ and $\bm e_m^{k+1}$ involve matrix multiplication with a vector operation. For more information, please refer to \cite{YLW}. More recently, \cite{FLY} used the slack variable representation to rewrite quantile loss, and develpo a parallel ADMM algorithm for penalized quantile regression. Despite the algorithm achieving good results in simulation, the design philosophy of the algorithm is the same as \cite{FLY}. Therefore, we will not provide further description.
\subsection{Parallel ADMM for solving the combined  regularization regression}
The \textit{unified  optimization form} for high-dimensional regression models with combined regularization  is defined as
\begin{align}\label{uof}
\mathop {\arg \min }\limits_{\boldsymbol{\beta}} \left\{ \mathcal L{{{\left(\boldsymbol{y}-\boldsymbol X \boldsymbol\beta  \right)}}}  + P_{\lambda_1}(|\bm \beta|)  +  P_{\lambda_2}(|\bm G \bm \beta|) \right \}, 
\end{align}
where $\bm G$ is  a matrix that changes with the method. For elastic-net regression and  group regression problems, $\bm G$ is the identity matrix, and for  fused lasso regression problems, $\bm G$ is $\bm F$.

Setting $\bm r_m = \boldsymbol{y}_m-\boldsymbol X_m \boldsymbol \beta_m $, $ \bm \beta_m = \bm \beta$ and $\bm G\bm \beta = \bm b$, the combined regularization regression models in  (\ref{uof}) can be rewritten as
\begin{align}\label{mcuof}
&\mathop {\arg \min }\limits_{\boldsymbol{\beta},  \bm b, \{\bm r\}_{m=1}^{M}, \{\bm \beta_m \}_{m=1}^{M}} \left\{ \sum\limits_{m=1}^{M}\mathcal L{{{\left( \bm r_m  \right)}}}  + P_{\lambda_1}(|\bm \beta|)  +  P_{\lambda_2}(|\bm b|) \right \},  \notag\\ 
& \textbf{s.t.} \ \bm X_m \bm \beta_m + \bm r_m = \bm y_m, \ \bm \beta_m = \bm \beta, \  \bm G \bm \beta = \bm b, \ m = 1,2,\dots,M. 
\end{align}

The augmented Lagrangian of (\ref{mcuof}) is given by
\begin{align}\label{alglasso}
 \sum\limits_{m=1}^{M}& \mathcal L{{{\left( \bm r_m  \right)}}}  + P_{\lambda_1}(|\bm \beta|)  +  P_{\lambda_2}(|\bm b|) -  \sum\limits_{m=1}^{M} \langle \bm  d_m, \bm X_m \bm \beta_m + \bm r_m - \bm y_m \rangle   \notag \\
& - \sum\limits_{m=1}^{M} \langle \bm  e_m, \bm \beta_m - \bm \beta \rangle  -  \langle \bm  f, \bm G \bm \beta - \bm b \rangle + \frac{\mu}{2} \sum\limits_{m=1}^{M} \|  \bm X_m \bm \beta_m + \bm r_m - \bm y_m \|_2^2 \\
& +  \frac{\mu}{2} \sum\limits_{m=1}^{M}  \| \bm \beta_m - \bm \beta \|_2^2 + \frac{\mu}{2} \|\bm G \bm \beta - \bm b \|_2^2, \notag
\end{align}
where $\bm d_m$, $\bm e_m$  and $\bm f$ are dual variables corresponding to the linear constraints, and $\mu>0$ is a given penalty parameter.

After rearranging the terms in (\ref{alglasso}) and omitting  constant terms, the update of the  variables in the iterative process can be written as
\begin{small}
\begin{equation}\label{primalupdate}
\left\{ \begin{aligned}
\boldsymbol \beta ^{k + 1} & \leftarrow \arg \min_{\boldsymbol\beta} \left\{  P_{\lambda_1}(|\bm \beta|) +   \frac{\mu}{2} \sum\limits_{m=1}^{M}  \| \bm \beta_m^k - \bm \beta  - \bm e_m^k/\mu\|_2^2 +   \frac{\mu}{2} \| \bm G \bm \beta - \bm b^k - \bm f^k/\mu \|_2^2  \right\};\\ %
{\boldsymbol{b}^{k + 1}} &  \leftarrow \arg \min_{\boldsymbol{b}} \left\{ P_{\lambda_2}(|\bm b|) +   \frac{\mu}{2} \| \bm G \bm \beta^{k+1} - \bm b - \bm f^k/\mu \|_2^2  \right\};  \\ %
{\boldsymbol{r}_m^{k + 1}} &  \leftarrow \arg \min_{\boldsymbol{r}_m}  \left\{ \mathcal L(\bm r_m) +  \frac{\mu}{2}\|  \bm X_m \bm \beta_m^k + \bm r_m - \bm y_m - \bm d_m^k/\mu \|_2^2  \right\},  m =1,2,\dots,M; \\ %
{\boldsymbol{\beta}_m^{k + 1}} &  \leftarrow \arg \min_{\boldsymbol{\beta}_m} \left\{ \frac{\mu}{2} \|  \bm X_m \bm \beta_m + \bm r_m^{k+1} - \bm y_m - \bm d_m^k/\mu \|_2^2 +  \frac{\mu}{2} \| \bm \beta_m - \bm \beta^{k+1}  - \bm e_m^k/\mu\|_2^2 \right\},  m =1,2,\dots,M; \\ %
\boldsymbol d_m ^{k + 1} & \leftarrow   \bm d_m ^k - \mu ( \bm X_m \bm \beta_m^{k+1} + \bm r_m^{k+1} - \bm y_m),   m =1,2,\dots,M;\\ %
\boldsymbol e_m ^{k + 1} & \leftarrow   \bm e_m ^k - \mu ( \bm \beta_m^{k+1} - \bm \beta^{k+1} ),  m =1,2,\dots,M;\\%
\boldsymbol f ^{k + 1} & \leftarrow   \bm f ^k - \mu (\bm G \bm \beta^{k+1} - \bm b^{k+1} ),  m =1,2,\dots,M.
\end{aligned} \right.
\end{equation}
\end{small}
From  (\ref{primalupdate}), we know that the central machine does not need to load data, it only needs to update $\bm \beta$, $\bm b$ and $\bm f$. The $m$-th local machine needs to load $\bm X_m$ and $\bm y_m$, and update its corresponding $\bm r_m$, $\bm \beta_m$, $\bm d_m$, and $\bm e_m$. We visualize the operations of the central and local machines in parallel algorithms in Figure \ref{Fig1}, and summarize the algorithm in Algorithm \ref{alg1}.

 In certain economic and financial applications, additional linear constraints are required on the coefficient $\bm \beta$, such as the sum constraint in investment arbitrage and the non-negative constraint in stock index tracking. In such cases, a minor modification is needed in the update of $\bm \beta^{k+1}$, that is,
\begin{align}
\boldsymbol \beta ^{k + 1} & \leftarrow \arg \min_{\boldsymbol\beta \in \mathcal{C}} \left\{  P_{\lambda_1}(|\bm \beta|) +   \frac{\mu}{2} \sum\limits_{m=1}^{M}  \| \bm \beta_m^k - \bm \beta  - \bm e_m^k/\mu\|_2^2 +   \frac{\mu}{2} \| \bm G \bm \beta - \bm b^k - \bm f^k/\mu \|_2^2  \right\},
\end{align}
 which does not change the structure of the algorithm. Here,  $\mathcal{C}$ is a space with linear constraints, such as $\bm 1^\top \bm \beta = 1$ and $\{\beta\}_{j=1}^p \ge 0$. In order to make our algorithm better suited for these specific applications, we include these linear constraints in the subsequent discussions.
\begin{figure}[H]
\centering
\includegraphics[width=16cm,height=8cm]{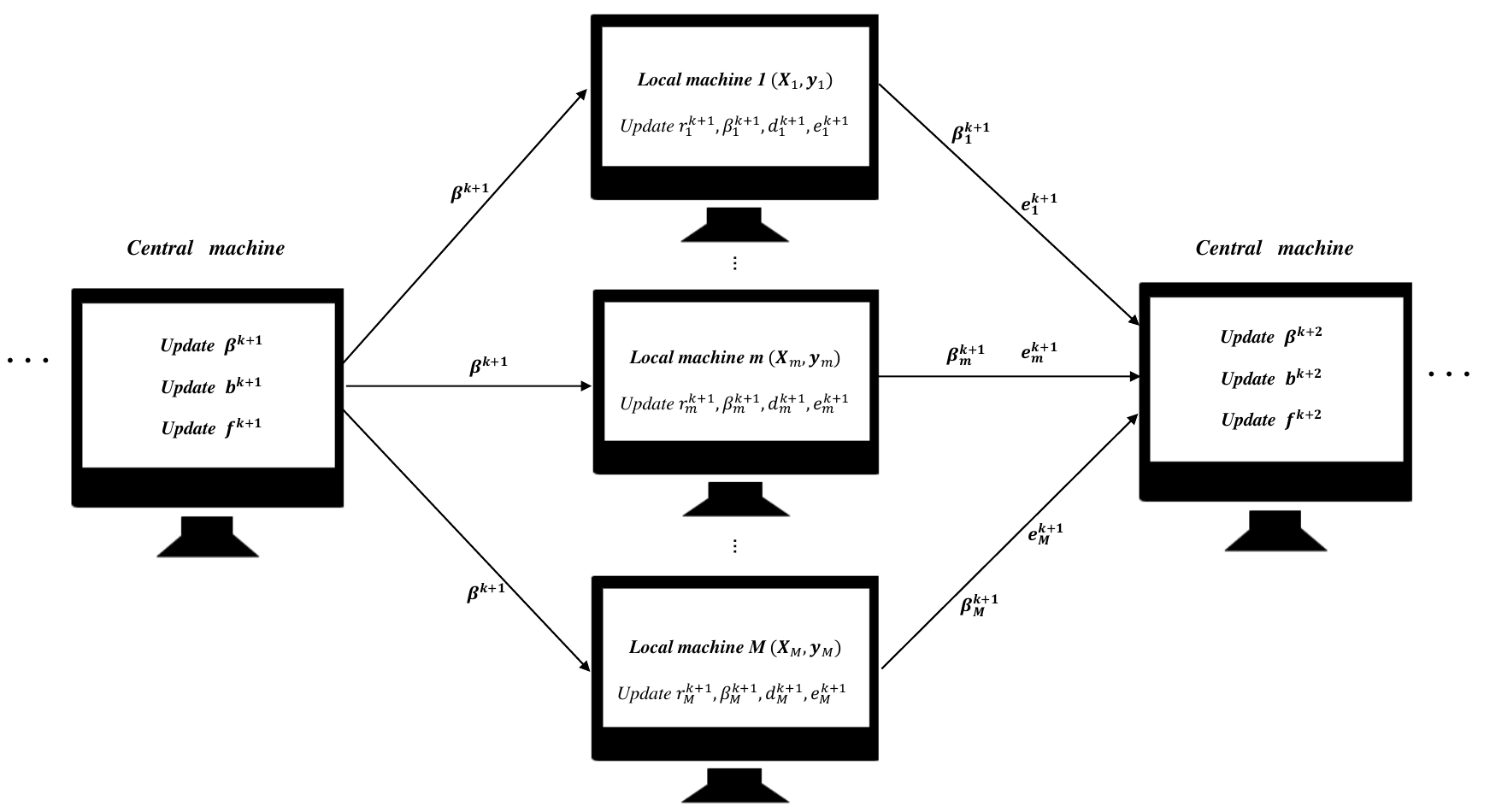}
\caption{\footnotesize{Schematic diagram of the implementation of parallel ADMM algorithm.}}\label{Fig1}
\end{figure}
\begin{algorithm}\small
\caption{\small{Parallel ADMM for solving the combined  regularization regression in  (\ref{glasso})}}
\label{alg1}
\begin{algorithmic}
\STATE {\textbf{Input:} $\bullet$ Central machine: $\mu, M, \lambda_1, \lambda_2,  \bm b^0, \bm f^0$.\\
\qquad \  \ \ \ \ $\bullet$ The $m$-th local machine:  $\{\boldsymbol{X}_m,\boldsymbol{y}_m\}_{m=1}^{M}$; $\mu$, $\delta$ (Huber loss), $\tau$ (quantile loss); $\bm \beta_m^0$, $\bm d_m^0$, $\bm e_m^0$.}
\STATE {\textbf{Output:} the total number of iterations $K$,  $\boldsymbol{\beta}^K$. }
\STATE {\textbf{while} not converged \textbf{do}}
\STATE {\ \textbf{Central machine}: 1. Receive $\bm \beta_m^k$ and $\bm e_m^k$ transmitted by $M$ local machines,\\ \qquad \qquad \qquad  \qquad \quad \ 2. Update $\bm \beta^{k+1}$ and $\bm b^{k+1}$, \\
\qquad \qquad \qquad  \qquad \quad \ 3. Update $\boldsymbol f ^{k + 1}  \leftarrow   \bm f ^k - \mu (\bm G\bm \beta^{k+1} - \bm b^{k+1} )$, \\
\qquad \qquad \qquad  \qquad \quad \ 4. Send $\bm \beta^{k+1}$ to the local machines.
}
\STATE {\ \textbf{Local machines}: \ \  for $m =1 ,2, \dots, M$ (in parallel) \\
\qquad \qquad \qquad  \qquad \quad \ 1. Receive $\bm \beta^{k+1}$  transmitted by the central machine, \\
\qquad \qquad \qquad  \qquad \quad \ 2. Update $\bm r_m^{k+1}$  and $\bm \beta_m^{k+1}$, \\
\qquad \qquad \qquad  \qquad \quad \ 3. Update $\boldsymbol d_m ^{k + 1}  \leftarrow   \bm d_m ^k - \mu ( \bm X_m \bm \beta_m^{k+1} + \bm r_m^{k+1} - \bm y_m)$ and $\boldsymbol e_m ^{k + 1}  \leftarrow   \bm e_m ^k - \mu ( \bm \beta_m^{k+1} - \bm \beta^{k+1} )$, \\
\qquad \qquad \qquad  \qquad \quad \ 4. Send $\bm \beta^{k+1}_m$ and $\bm e^{k+1}_m$  to the central machine.
}
\STATE {\textbf{end while}}
\STATE {\textbf{return} solution}.
\end{algorithmic}
\end{algorithm}
\section{The specific implementation of parallel ADMM algorithm}\label{sec4}
\qquad From (\ref{primalupdate}) and  Algorithm \ref{alg1}, it can be observed that the operations on each local machine do not need to vary with changes in the regularization term. In other words, we only need to modify the operations on the central machine to solve high-dimensional regression models with different combined regularizers. Here, 
we first discuss the updates of  local machines (in parallel), and then discuss the updates of the central machine separately according to three types of combined regularization terms. Note that we only need to focus on updating the primal variables, while the updates of the dual variables, $\bm d_m$, $\bm e_m$ and $\bm f$ involve linear algebraic operations and are easy to implement. 

For the update of the subproblem $\bm r^{k+1}_m$  in (\ref{primalupdate}), it is evidently an  $\mathcal L$ proximal operator ($\text{prox}_{\frac{1}{n} \mathcal L }( \bm y_m + \bm d_m^k/\mu - \bm X_m \bm \beta_m^k)$), and its closed-form solution can be obtained from Table \ref{Tab2}.
\begin{align}\label{rm}
\bm r_m^{k+1} \leftarrow \text{prox}_{\frac{1}{n} \mathcal L }( \bm y_m + \bm d_m^k/\mu - \bm X_m \bm \beta_m^k), \ m=1,2,\dots,M.
\end{align}

For the update of the subproblem $\bm \beta^{k+1}_m$  in (\ref{primalupdate}), the minimization problem is quadratic and differentiable, allowing us to solve the subproblem by solving the following linear equations:
\begin{align}\label{betam}
\bm \beta_m \leftarrow (\bm X_m^T \bm X_m + \bm I_p)^{-1} \left[ \bm X_m^T(  \bm y_m + {\bm d_m^k}/{\mu} - \bm r_m^{k+1})  + (\bm \beta^{k+1} + {\bm e_m^k}/{\mu}) \right].
\end{align}
When $p$ is large than $n_m$, \cite{YLW}  suggested using the Woodbury matrix identity $(\bm X_m^T \bm X_m + \bm I_p)^{-1} =  \bm I_p - \bm X_m^\top ( \bm I_{n_m} + \bm X_m \bm X_m^\top)^{-1} \bm X_m$. This method is actually very practical when the size of $n_m$ is small because throughout the iteration process of the entire algorithm, the inverse only needs to be computed once. However, when both $n_m$ and $p$ are large, computing the inverse can be time-consuming, and in such cases, the conjugate gradient method  performs better in solving the equation in (\ref{betam}). The conjugate gradient method is an efficient iterative algorithm used to solve symmetric positive definite linear systems of equations, and it has advantages in handling large-scale sparse problems.

Next, we  will provide a detailed description of the update steps for the central machine for different combined regularization terms.

\subsection{Elastic-net}\label{sec41}
In this section,  we describe the updates of each subproblem in central machine for implementing  regression with  elastic-net regularizations. Then, in (\ref{mcuof}), $\bm G = \bm I_p$, $P_{\lambda_1}(|\bm \beta|) =  \lambda_1 \|\bm \beta \|_1$
and $P_{\lambda_2}(|\bm b|) = \lambda_2 \|\bm b \|_2^2$.

For the subproblem of updating $\bm \beta^{k+1}$ in (\ref{primalupdate}), by rearranging the optimization equation and omitting some constant terms that are not relevant to the optimization target variable $\bm \beta$, the following equation can be obtained,
\begin{align}\label{stbeta}
\bm \beta^{k+1} \leftarrow \arg \min_{\boldsymbol\beta} \left\{ {\lambda}_1 \left \| {{\boldsymbol{\beta}}} \right \|_1 + \frac{\mu(M+1)}{2} \left\| \bm \beta - \frac{M(\bar{\bm \beta}^k - \bar{\bm e}^k/ \mu) + (\bm b^k + \bm f^k/\mu)}{M+1} \right\|_2^2 \right\},
\end{align}
where $\bar{\bm \beta}^k= M^{-1}  \sum_{m=1}^{M} \bm \beta_m^k$ and $\bar{\bm e}^k= M^{-1}  \sum_{m=1}^{M} \bm e_m^k$. As a result, the closed-form solution of the soft-thresholding operator in  (\ref{stbeta})  is given as follows:
\small{
\begin{align}\label{enbeta}
\bm \beta^{k+1} \leftarrow \text{sign}\left(\frac{M(\bar{\bm \beta}^k - \bar{\bm e}^k/ \mu) + (\bm b^k + \bm f^k/\mu)}{M+1}\right) \odot \left( \left|\frac{M(\bar{\bm \beta}^k - \bar{\bm e}^k/ \mu) + (\bm b^k + \bm f^k/\mu)}{M+1} \right| - \frac{{\lambda}_1}{\mu(M+1)}\right)_+.
\end{align}}
Here,  $\odot$ denotes Hadamard product,   $\text{sign}(\cdot)$ represents the sign function, and $(\cdot)_+$ denotes the operation of taking the positive part element-wise. Note when the variable $\bm \beta$ needs to satisfy linear constraints ($\bm \beta \in \mathcal{C}$),  we just need to project $\bm \beta^{k+1}$ in (\ref{enbeta}) onto the linear space $\mathcal{C}$. Mathematically, the projection of vector $\bm \beta^{k+1}$ onto linear constraint space $\mathcal{C}$ can be represented as,
\begin{align}\label{proj}
\bm \beta_{\mathcal{C}}^{k+1}  = \mathrm{Proj}_\mathcal{C}(\bm{\beta}^{k+1}) = \mathrm{arg} \min_{\bm{\beta} \in \mathcal{C}} \| \bm{\beta} - \bm{\beta}^{k+1} \|_2.
\end{align}
Projection onto convex sets is a well-studied concept. For the applications mentioned in this paper, the projection can be analytically solved, see Section 15.2 of \cite{L2013} for several examples.
Therefore, when addressing regression problems with linear constraints, we can substitute $\bm \beta_{\mathcal{C}}^{k+1}$ for $\bm \beta^{k+1}$ in Algorithm \ref{alg1} to iterate over $\bm \beta$.

For the subproblem of updating $\bm b^{k+1}$ in (\ref{primalupdate}), the minimization problem is quadratic and differentiable, allowing us to solve the subproblem by
\begin{align}
\bm b^{k+1} = \frac{\mu(\bm \beta^{k+1} - \bm f^k/\mu)}{2\lambda_2 + \mu}.
\end{align}

\subsection{Sparse group lasso}\label{sec42}
In this section,  we describe the updates of each subproblem in central machine for implementing  regression with  sparse group lasso. Then, in (\ref{mcuof}), $\bm G = \bm I_p$, $P_{\lambda_1}(|\bm \beta|) =  \lambda_1 \|\bm \beta \|_1$
and $P_{\lambda_2}(|\bm b|) = \lambda_2 \|\bm b \|_{2,1}$.

For the subproblem of updating $\bm \beta^{k+1}$ in (\ref{primalupdate}),  it follows the same procedure as described in Section \ref{sec41}, where it is updated as
\small{
\begin{align}\label{grbeta}
\bm \beta^{k+1} \leftarrow \text{sign}\left(\frac{M(\bar{\bm \beta}^k - \bar{\bm e}^k/ \mu) + (\bm b^k + \bm f^k/\mu)}{M+1}\right) \odot \left( \left|\frac{M(\bar{\bm \beta}^k - \bar{\bm e}^k/ \mu) + (\bm b^k + \bm f^k/\mu)}{M+1} \right| - \frac{{\lambda}_1}{\mu(M+1)}\right)_+,
\end{align}}
where $\bar{\bm \beta}^k= M^{-1}  \sum_{m=1}^{M} \bm \beta_m^k$ and $\bar{\bm e}^k= M^{-1}  \sum_{m=1}^{M} \bm e_m^k$. For constrained regression problems, the same strategy as (\ref{proj}) can be used.

For the update of the subproblem $\bm b^{k+1}$, it is evidently an $\ell_{2,1}$-norm proximal operator ($\text{prox}_{\lambda_2 \| \cdot \|_{2,1}}(\bm{\beta}^{k+1}- {\bm{f}^k}/{\mu} )$), and its closed-form solution is as follows:
\begin{align}\label{gb}
\bm b_g^{k+1} \leftarrow  \frac{ \bm{\beta}^{k+1}_g- {\bm{f}^k}_g/{\mu} }{\| \bm{\beta}^{k+1}_g- {\bm{f}^k}_g/{\mu}  \|_2} \cdot (\|\bm{\beta}^{k+1}_g- {\bm{f}^k}_g/{\mu}  \|_2 - \lambda_2/\mu)_+, \ g=1,2,\dots,G.
\end{align}

\subsection{Sparse fused lasso}
In this section,  we describe the updates of each subproblem in central machine for implementing  regression with  sparse group lasso. Then, in (\ref{mcuof}), $\bm G = \bm F$, $P_{\lambda_1}(|\bm \beta|) =  \lambda_1 \|\bm \beta \|_1$
and $P_{\lambda_2}(|\bm b|) = \lambda_2 \|\bm b \|_{1}$.
Due to the presence of the $\bm F$ matrix, the update of $\bm \beta^{k+1}$ in the central machine differs from that in Section \ref{sec41} and Section \ref{sec42}.

 For the subproblem of updating $\bm \beta^{k+1}$, by rearranging the optimization equation and omitting some constant terms that are not relevant to the optimization target variable $\bm \beta$, the following equation can be obtained,
\begin{align}\label{sgbeta}
\bm \beta^{k+1} \leftarrow \arg \min_{\boldsymbol\beta} \left\{ {\lambda}_1 \left \| {{\boldsymbol{\beta}}} \right \|_1 + \frac{\mu}{2} \bm \beta^\top (M \bm I_p + \bm F^\top \bm F) \bm \beta - \mu \bm \beta^\top\left[ M(\bar{\bm \beta}^k - \bar{\bm e}^k/\mu)  + \bm F^\top (\bm b^k + \bm f^k/\mu  ) \right]  \right\},
\end{align}
where $\bar{\bm \beta}^k= M^{-1}  \sum_{m=1}^{M} \bm \beta_m^k$ and $\bar{\bm e}^k= M^{-1}  \sum_{m=1}^{M} \bm e_m^k$. Clearly, due to the non-identity matrix before the quadratic term of $\beta$ and $\left \| {{\boldsymbol{\beta}}} \right \|_1$, the equation (\ref{sgbeta}) does not have a closed-form solution. Although we can solve  the equation (\ref{sgbeta})  using numerical methods such as the coordinate descent method, it would  significantly increase the computational burden.  Here, we suggest using a linearization method to approximate this optimization problem and obtain a closed-form solution for $\bm \beta^{k+1}$. Specifically, we propose to add a proximal term to the objective function in Equation (\ref{sgbeta}) and update $\bm \beta^{k+1}$ using the following formula,
\small{
\begin{align}\label{lbeta}
\bm \beta^{k+1} \leftarrow \arg \min_{\boldsymbol\beta} \left\{ {\lambda}_1 \left \| {{\boldsymbol{\beta}}} \right \|_1 + \frac{\mu}{2} \bm \beta^\top (M \bm I_p + \bm F^\top \bm F) \bm \beta - \mu \bm \beta^\top\left[ M(\bar{\bm \beta}^k - {\bar{\bm e}^k}/{\mu})  + \bm F^\top (\bm b^k + {\bm f^k}/{\mu}  ) \right]  + \frac{1}{2} \| \bm \beta - \bm \beta^k   \|_{S}^2 \right\},
\end{align}}
where $\bm S = \eta \bm I_p - (M \bm I_p + \bm F^\top \bm F)$ is a positive semidefinite matrix. To ensure that $\bm S$ is a positive definite matrix, $\eta$ must be greater than or equal to the largest eigenvalue of $ (M \bm I_p + \bm F^\top \bm F)$. It is worth noting that the largest eigenvalue of $ (M \bm I_p + \bm F^\top \bm F)$ is equal to $M+4$. Then, we can set $\eta = M+4$.  After rearranging the terms in (\ref{lbeta}) and omitting  constant terms, the update of the  $\bm \beta^{k+1}$ can be written as
\begin{align}\label{lbeta2}
\bm \beta^{k+1} \leftarrow \arg \min_{\boldsymbol\beta} \left\{ {\lambda}_1 \left \| {{\boldsymbol{\beta}}} \right \|_1 + \frac{\eta}{2} \left\|\bm \beta -  \bm \beta^k + \frac{\mu}{\eta} \left[ (M \bm I_p + \bm F^\top \bm F)\bm \beta^k - M(\bar{\bm \beta}^k - {\bar{\bm e}^k}/{\mu})  - \bm F^\top (\bm b^k + {\bm f^k}/{\mu}  )    \right]  \right\|_2^2      \right\}.
\end{align}
As a result, the closed-form solution of the soft-thresholding operator in  (\ref{lbeta2})  is given as follows:
\small{
\begin{align}\label{fubeta}
\bm \beta^{k+1} \leftarrow \text{sign}\left(  \bm \beta^k - {\mu}/{\eta} \left[ (M \bm I_p + \bm F^\top \bm F)\bm \beta^k - M(\bar{\bm \beta}^k - {\bar{\bm e}^k}/{\mu})  - \bm F^\top (\bm b^k + {\bm f^k}/{\mu}  )    \right]       \right) \odot  \\\notag
\left( \left| \bm \beta^k - {\mu}/{\eta} \left[ (M \bm I_p + \bm F^\top \bm F)\bm \beta^k - M(\bar{\bm \beta}^k - {\bar{\bm e}^k}/{\mu})  - \bm F^\top (\bm b^k + {\bm f^k}/{\mu}  )    \right]    \right| - {{\lambda}_1}/{\eta}\right)_+.
\end{align}}
For constrained regression problems, we can employ the same strategy as (\ref{proj}).

\qquad For the update of the subproblem $\bm b^{k+1}$, it is also a soft-thresholding operator ($\text{prox}_{\lambda_1 \| \cdot \|_{1}}(\bm F \bm{\beta}^{k+1}- {\bm{f}^k}/{\mu} )$), and its closed-form solution is as follows:
\begin{align}\label{fb}
\bm b^{k+1} \leftarrow  \text{sign}\left( \bm F \bm{\beta}^{k+1}- {\bm{f}^k}/{\mu}    \right) \odot  \left(  \left| \bm F \bm{\beta}^{k+1}- {\bm{f}^k}/{\mu}   \right| -    \lambda_2/\mu  \right)_+.
\end{align}

\subsection{Nonconvex extension}
Compared to the combination of convex regularizers, the main difference of the combination of nonconvex regularizers lies in $P_{\lambda_1}(|\bm \beta|)$. In this paper, we mainly consider two popular non-convex regularizers, scad and mcp.  According to the suggestion in \cite{ZL}, we can use  a   unified  method  named local  linear  approximation  to  handle  the  nonconvex penalty, that is
\begin{align}\label{one}
P_{\lambda_1}(|\bm \beta|) \approx  P_{\lambda_1}(|\bm \beta^l|) + \nabla P_{\lambda_1}(|\bm \beta^l|)^T (|\bm \beta| - |\bm \beta^l|), \ \text{for} \ \bm \beta  \approx \bm \beta^l,
\end{align}
where  $\bm \beta^l$ is the solution from the last iteration, and $\nabla P_{\lambda_1}(|\bm \beta^l|) = (\nabla P_\lambda(|\bm \beta_1^l|), \nabla P_\lambda(|\bm \beta_2^l|), \dots, \nabla P_\lambda(|\bm \beta_p^l|)) ^\top$.

$\bullet$ For scad, we have \begin{align}\label{scad}
\nabla P_\lambda(|\bm \beta_j|) = \begin{cases}
\lambda_1, & \text{{if }} |\bm \beta_j| \leq \lambda_1, \\
\frac{a\lambda_1 -  |\bm \beta_j|}{a-1}, & \text{{if }} \lambda_1 < |\bm \beta_j|  < a  \lambda_1, \\
0 , & \text{{if }} |\bm \beta_j| \ge a  \lambda_1. \\
\end{cases}
\end{align}

$\bullet$ For mcp, we have \begin{align}\label{mcp}
\nabla P_\lambda(|\bm \beta_j|)  =
\begin{cases}
\lambda_1  - \frac{|\bm \beta_j|}{a}, & \text{{if }}  |\bm \beta_j|  \le a \lambda_1\\
0, & \text{{if }}  |\bm \beta_j|  > a \lambda_1 \\
\end{cases}
\end{align}
For nonconvex combined  regularization, it can be written as
\begin{align}\label{41}
\mathop {\arg \min }\limits_{\boldsymbol{\beta}} \left\{ \mathcal L{{{\left(\boldsymbol{y}-\boldsymbol X \boldsymbol\beta  \right)}}}  + P_{\lambda_1}(|\bm \beta|)  +  P_{\lambda_2}(\bm G \bm \beta) \right \}. 
\end{align}
where  $P_{\lambda_1}(|\bm \beta|) = \text{scad}(\bm \beta)  \ \text{or}  \  \text{mcp}(\bm \beta) $. By substituting equation (\ref{one}) into equation (\ref{41}), we can obtain the following optimized form in a weighted manner,
\begin{align}\label{we}
\mathop {\arg \min }\limits_{\boldsymbol{\beta}^{l+1}} \left\{ \mathcal L{{{\left(\boldsymbol{y}-\boldsymbol X \boldsymbol\beta  \right)}}}  + \sum_{j=1}^{p}\nabla P_{\lambda_1}(|\beta_j^{l}|)|\beta_j|  +  P_{\lambda_2}(\bm G \bm \beta) \right \}.
\end{align}
Note that we only need to make a small change to solve this weighted combined optimization form using Algorithm \ref{alg1}. This change only requires replacing $\lambda_1$ in (\ref{enbeta}), (\ref{grbeta}), and (\ref{fubeta}) with $\nabla P_{\lambda_1}(|\bm \beta^{l}|)$. 

To solve nonconvex regression using the LLA algorithm, it is necessary to find a good initial value. As suggested by Gu2018, we can use the solution of $P_{\lambda_1}(|\bm \beta|) = \lambda_1 \|\bm \beta \|_1 $ in  (\ref{41}) as the initial value. Then, we get the solution of (\ref{41}) by  solving a sequence of weighted $\ell_1$-penalized
quantile regression.  The LLA algorithm solves the nonconvex  regression via the following iterations: 
\begin{algorithm}\small
\caption{\small{The local linear approximation (LLA) algorithm for combined nonconvex  regression}}
\label{alg2}
\begin{algorithmic}
\STATE {1. Initialize $\bm \beta$ with  $\bm \beta^1$, where $\bm \beta^1$ is obtained by Algorithm \ref{alg1}.}
\STATE {2. For $l=1,2,\dots,L$, continue iterating the LLA iteration until convergence is achieved. }
\STATE {\quad \ 2.1. Compute the weights  $\nabla P_{\lambda_1}(|\bm \beta^{l}|)=(\nabla P_\lambda(|\bm \beta_1^{l}|), \nabla P_\lambda(|\bm \beta_2^{l}|), \dots, \nabla P_\lambda(|\bm \beta_p^{l}|)) ^\top$ by (\ref{scad}) or  (\ref{mcp})}
\STATE {\quad \ 2.2. Solve the weighted problem in (\ref{we}) by modified Algorithm \ref{alg1} with $\lambda_1$ in (\ref{enbeta}), (\ref{grbeta}), and (\ref{fubeta}) replacing with $\nabla P_{\lambda_1}(|\bm \beta^{l-1}|)$. Let this solution be denoted as $\bm \beta^{l+1}$.}
\end{algorithmic}
\end{algorithm}
It can be seen that the nonconvex combined  regularization regression  is solved through multiple iterations of convex combined  regularization regression.  Moreover,  \cite{FLY2014} demonstrated that, theoretically, only two or three iterations are sufficient to obtain a solution with high statistical accuracy.

\section{Convergence Analysis and Algorithm Complexity}\label{sec5}
\qquad  The traditional ADMM algorithm is commonly used to solve problems that have two-block separable objective functions and are connected by equality constraints. Two-block separable function refers to having two primal variables that need to be optimized alternately. For a comprehensive overview of the ADMM algorithm, one may refer to the cited reference \cite{SNEBJ} and the related literature mentioned therein.
It is evident that when there are $M$ local machines, both Algorithm \ref{alg1}  involves a total of $2M+2$ primal variables, namely $\left\{\bm \beta, \bm b, \{\bm r_m, \bm \beta_m  \}_{m=1}^{M}  \right\}$. Thus, the parallel ADMM algorithms can be categorized as multi-block ADMM algorithms. However, a recent study by \cite{CBYY} has demonstrated that directly extending the ADMM algorithm for convex optimization with three or more separable blocks may not guarantee convergence, and they even provided an example of divergence. Fortunately, they also established a sufficient condition to ensure convergence for the direct extension of multi-block ADMM algorithm. Their findings suggest that convergence of multi-block ADMM algorithm is guaranteed as long as certain constrained coefficient matrices are orthogonal. This allows the iterations of all primal subproblems to be sequentially divided into two independent parts, ensuring convergence. In the following, we will demonstrate that the two  algorithms can be transformed into traditional  two-block ADMM algorithms.

To simplify notation, we here discuss the convergence of the parallel algorithms in the case of two local machines. The case of multiple local machines ($M>2$) is similar to the case of two local machines.  In this way, the collected data will be divided into two parts,
\begin{align}
\bm X = \begin{bmatrix}
\bm X_{1} \\
\bm X_{2} 
\end{bmatrix}
\
\text{and}
\
\bm y = \begin{bmatrix}
\bm y_{1} \\
\bm y_{2} 
\end{bmatrix}.
\end{align}

\subsection{The Convergence of Algorithm \ref{alg1}}
Considering the presence of  two local machines ($M=2$), in order to accommodate the parallel structure, we need to introduce the terms $\{ \bm r_m = \bm y_m - \bm X_m \bm \beta_m \}_{m=1}^{2}$ and $\{ \bm \beta_m = \bm \beta \}_{m=1}^2$.
Revisiting the constrained optimization form in (\ref{mcuof}) and taking into account the linear constants $\bm \beta \in \mathcal{C}$, we obtain the following optimization problem:
\begin{align}\label{proofg1}
&\mathop {\arg \min }\limits_{\boldsymbol{\beta}, \bm b, \bm r_1, \bm r_2, \bm \beta_1, \bm \beta_2 } \left\{ \sum\limits_{m=1}^{2}\mathcal L{{{\left( \bm r_m  \right)}}}  +  P_{\lambda_1}(|\bm \beta|)  +  P_{\lambda_2}(|\bm b|) + I_{\mathcal{C}}(\bm \beta) \right \},  \notag\\ 
& \textbf{s.t.} \ \bm X_m \bm \beta_m + \bm r_m = \bm y_m, \ \bm \beta_m = \bm \beta, \  \bm G \bm \beta = \bm b, \ m = 1,2,
\end{align}
where 
\[
I_{\mathcal{C}}(\bm{\beta}) = \begin{cases}
0, & \text{if } \bm{\beta} \in \mathcal{C}, \\
\infty, & \text{if } \bm{\beta} \notin \mathcal{C}.
\end{cases}
\]

Let $\theta_1(\boldsymbol\beta)=  P_{\lambda_1}(|\bm \beta|) + I_{\mathcal{C}}(\bm{\beta})$,  $\theta_2(\boldsymbol r_1)= \mathcal L{{{\left( \bm r_1  \right)}}} $,   $\theta_3(\boldsymbol r_2)= \mathcal L{{{\left( \bm r_2  \right)}}} $, $\theta_4(\boldsymbol \beta_1)= 0 $, $\theta_5(\boldsymbol \beta_2)=0$, $\theta_6(\boldsymbol b)= P_{\lambda_2}(|\bm b|)$, where $\mathcal L$ only needs to be a convex loss function, including the least squares loss, quantile loss, square root loss, and Huber loss considered in this paper. Note that both  $P_{\lambda_1}(|\bm \beta|)$, $P_{\lambda_2}(|\bm b|)$ and $I_{\mathcal{C}}(\bm \beta)$ are convex, even if $P_{\lambda_1}(|\bm \beta|)$ is a weighted version for nonconvex regularization. Clearly, the optimization equation (\ref{proofg1}) involves six optimization variables (primal variables in the ADMM algorithm). Therefore, in the case of two local machines, the proposed parallel ADMM algorithm can be decomposed into six-block  ADMM algorithm.

Let
\begin{align}\label{IG}
\bm I_{\bm G}  =
\begin{cases}
\bm I_p, & \text{{if }}  \bm G = \bm I_p,\\
\bm I_{p-1}, & \text{{if }}  \bm G = \bm F, \\
\end{cases}
\end{align}
and set  
$$
{\boldsymbol{A}_1} = \begin{bmatrix}
\boldsymbol G \\
-\boldsymbol{I}_p \\
-\boldsymbol{I}_p \\
\bm 0 \\
\bm 0
\end{bmatrix},
\quad
{\boldsymbol{A}_2} = \begin{bmatrix}
\bm 0 \\
\bm 0 \\
\bm 0 \\
\bm I_n \\
\bm 0
\end{bmatrix},
\quad
{\boldsymbol{A}_3} = \begin{bmatrix}
\bm 0 \\
\bm 0 \\
\bm 0 \\
\bm 0 \\
\bm I_n
\end{bmatrix},
$$
$$
\quad
{\boldsymbol{B}_1} = \begin{bmatrix}
\bm 0 \\
\boldsymbol{I}_p \\
\bm 0 \\
\bm X_1 \\
\bm 0
\end{bmatrix},
\quad
{\boldsymbol{B}_2} = \begin{bmatrix}
\bm 0 \\
\bm 0 \\
\boldsymbol{I}_p \\
\bm 0 \\
\bm X_2
\end{bmatrix},
\quad
{\boldsymbol{B}_3} = \begin{bmatrix}
-\boldsymbol{I}_{\bm G} \\
\bm 0 \\
\boldsymbol{0} \\
\bm 0 \\
\bm 0
\end{bmatrix}
\quad
\text{and} \quad  \boldsymbol{c} = \begin{bmatrix}
\boldsymbol{0} \\
\boldsymbol{0} \\
\boldsymbol{0} \\
\boldsymbol{y}_1 \\
\boldsymbol{y}_2
\end{bmatrix},
$$
and rearrange the equations in (\ref{proofg1}), thus the  constrained optimization form in (\ref{proofg1}) can be rewritten as
\begin{align}\label{ggcons}
& \mathop  {\arg \min }\limits_{\boldsymbol{\beta}, \bm r_1, \bm r_2, \bm \beta_1, \bm \beta_2, \bm b} \big\{ \theta_1(\bm \beta)  + \theta_2(\bm r_1) + \theta_3(\bm r_2) + \theta_4(\bm \beta_1) + \theta_5(\bm \beta_2) + \theta_6(\bm b)  \big \},  \notag\\ 
& \textbf{s.t.} \  \bm A_1 \bm \beta + \bm A_2 \bm r_1 +  \bm A_3 \bm r_2 + \bm B_1 \bm \beta_1 + \bm B_2 \bm \beta_2 + \bm B_3 \bm b = \bm c.
\end{align}
Note that  in (\ref{primalupdate}), $\bm b^{k+1}$  is updated as the second variable among all the primal variables, while in equation (\ref{ggcons}) it is listed last. In fact, the update of $\bm b$ depends solely on $\bm \beta$ and is independent of $\bm r_1$,  $\bm r_2$, $\bm \beta_1$ and $\bm \beta_2$. Therefore, in all the updates of the primal variables, $\bm b^{k+1}$  only needs to be placed after $\bm \beta^{k+1}$; placing it in other positions will not affect the solution of  Algorithm \ref{alg1}. This is also reflected in the relationship between the matrices, clearly, $\bm B_3$ is orthogonal to $\bm A_2, \bm A_3, \bm B_1$ and $\bm B_2$.

Taking into account the linearized term of $\bm \beta$, the augmented Lagrangian form of  (\ref{ggcons})  is
\small{
\begin{align}\label{4.2}
\theta_1(\bm \beta) & + \theta_2(\bm r_1) + \theta_3(\bm r_2) + \theta_4(\bm \beta_1) + \theta_5(\bm \beta_2) + \theta_6(\bm b) - \bm z^\top(\bm A_1 \bm \beta + \bm A_2 \bm r_1 +  \bm A_3 \bm r_2 + \bm B_1 \bm \beta_1 + \bm B_2 \bm \beta_2 + \bm B_3 \bm b - \bm c) \notag \\
 &+ \frac{\mu}{2}\| \bm A_1 \bm \beta + \bm A_2 \bm r_1 +  \bm A_3 \bm r_2 + \bm B_1 \bm \beta_1 + \bm B_2 \bm \beta_2 + \bm B_3 \bm b - \bm c\|_2^2 + \frac{1}{2}\|\bm \beta - \bm \beta^k\|_{ \bm S_G}^{2},
\end{align}}
where $\bm z = (\bm f^\top, \bm d_1^\top, \bm d_2^\top, \bm e_1^\top, \bm e_2^\top)^\top$, and
\begin{align}\label{IG}
\bm S_{\bm G}  =
\begin{cases}
\bm 0, & \text{{if }}  \bm G = \bm I_p,\\
\bm S, & \text{{if }}  \bm G = \bm F, \\
\end{cases}
\end{align}
where $\bm S = \eta \bm I_p - (M \bm I_p + \bm F^\top \bm F)$.
According to the first-order optimality conditions of the minimization problems in (\ref{4.2}), we have 
\begin{footnotesize}
\begin{equation}\label{6pri}
\left\{
\begin{array}{l}
\theta_1(\boldsymbol\beta)-\theta_1(\boldsymbol{\beta}^{k+1})+(\boldsymbol\beta-\boldsymbol{\beta}^{k+1})^\top \left\{-\boldsymbol{A_1}^\top [ \boldsymbol{z}^k-\mu (\bm A_1 \bm \beta^{k+1}  + \bm A_2 \bm r_1^k +  \bm A_3 \bm r_2^k + \bm B_1 \bm \beta_1^k + \bm B_2 \bm \beta_2^k   \right. \\
\left.  + \bm B_3 \bm b^k - \bm c) ] + \bm S_{\bm G}(\bm \beta - \bm \beta^k) \right\}\ge 0, \\%
\theta_2(\boldsymbol  r_1)-\theta_2(\boldsymbol{r}_1^{k+1})+(\boldsymbol r_1-\boldsymbol{r}_1^{k+1})^\top \left\{-\boldsymbol{A_2}^\top [ \boldsymbol{z}^k-\mu(\bm A_1 \bm \beta^{k+1} + \bm A_2 \bm r_1^{k+1} +  \bm A_3 \bm r_2^k + \bm B_1 \bm \beta_1^k + \bm B_2 \bm \beta_2^k  \right. \\
\left. + \bm B_3 \bm b^k - \bm c) ] \right\}\ge 0,\\%
\theta_3(\boldsymbol r_2 )-\theta_3(\boldsymbol{r}_2^{k+1})+(\boldsymbol r_2-\boldsymbol{r}_2^{k+1})^\top \left\{-\boldsymbol{A_3}^\top [ \boldsymbol{z}^k-\mu(\bm A_1 \bm \beta^{k+1} + \bm A_2 \bm r_1^{k+1} +  \bm A_3 \bm r_2^{k+1} + \bm B_1 \bm \beta_1^k + \bm B_2 \bm \beta_2^k  \right. \\
\left.  + \bm B_3 \bm b^k - \bm c) ] \right\}\ge 0,\\%
\theta_4(\boldsymbol \beta_1 )-\theta_4(\boldsymbol{\beta}_1^{k+1})+(\boldsymbol \beta_1-\boldsymbol{\beta}_1^{k+1})^\top \left\{-\boldsymbol{B_1}^\top [ \boldsymbol{z}^k-\mu(\bm A_1 \bm \beta^{k+1} + \bm A_2 \bm r_1^{k+1} +  \bm A_3 \bm r_2^{k+1} + \bm B_1 \bm \beta_1^{k+1} + \bm B_2 \bm \beta_2^k   \right. \\
\left. + \bm B_3 \bm b^k - \bm c) ] \right\}\ge 0, \\%
\theta_5(\boldsymbol \beta_2 )-\theta_5(\boldsymbol{\beta}_2^{k+1})+(\boldsymbol \beta_2-\boldsymbol{\beta}_2^{k+1})^\top \left\{-\boldsymbol{B_2}^\top [ \boldsymbol{z}^k-\mu(\bm A_1 \bm \beta^{k+1} + \bm A_2 \bm r_1^{k+1} +  \bm A_3 \bm r_2^{k+1} + \bm B_1 \bm \beta_1^{k+1} + \bm B_2 \bm \beta_2^{k+1}  \right. \\
\left.  + \bm B_3 \bm b^k - \bm c) ] \right\}\ge 0, \\%
\theta_6(\boldsymbol b )-\theta_6(\boldsymbol{b}^{k+1})+(\boldsymbol b-\boldsymbol{b}_2^{k+1})^\top \left\{-\boldsymbol{B_3}^\top [ \boldsymbol{z}^k-\mu(\bm A_1 \bm \beta^{k+1} + \bm A_2 \bm r_1^{k+1} +  \bm A_3 \bm r_2^{k+1} + \bm B_1 \bm \beta_1^{k+1} + \bm B_2 \bm \beta_2^{k+1}   \right. \\
\left. + \bm B_3 \bm b^{k+1} - \bm c) ] \right\}\ge 0.
\end{array}\right.
\end{equation}
\end{footnotesize}
It is easy to verify that
\begin{align*}\label{ort-con}
\boldsymbol{A_1}^\top \boldsymbol{A_2}=\boldsymbol 0, \ \boldsymbol{A_1}^\top \boldsymbol{A_3}=\boldsymbol 0 \ \text{and} \ \boldsymbol{A_2}^\top \boldsymbol{A_3}=\boldsymbol 0,  \\
\boldsymbol{B_1}^\top \boldsymbol{B_2}=\boldsymbol 0, \ \boldsymbol{B_1}^\top \boldsymbol{B_3}=\boldsymbol 0 \ \text{and} \ \boldsymbol{B_2}^\top \boldsymbol{B_3}=\boldsymbol 0.
\end{align*}
Together with (\ref{6pri}),  we have 
\begin{footnotesize}
\begin{equation}\label{6pri2}
\left\{
\begin{array}{l}
\theta_1(\boldsymbol\beta)-\theta_1(\boldsymbol{\beta}^{k+1})+(\boldsymbol\beta-\boldsymbol{\beta}^{k+1})^\top \left\{-\boldsymbol{A_1}^\top \left[ \boldsymbol{z}^k-\mu (\bm A_1 \bm \beta^{k+1}  +  \bm B_1 \bm \beta_1^k + \bm B_2 \bm \beta_2^k + \bm B_3 \bm b^k - \bm c) \right] + \bm S_{\bm G}(\bm \beta - \bm \beta^k) \right\}\ge 0, \\%
\theta_2(\boldsymbol  r_1)-\theta_2(\boldsymbol{r}_1^{k+1})+(\boldsymbol r_1-\boldsymbol{r}_1^{k+1})^\top \left\{-\boldsymbol{A_2}^\top \left[ \boldsymbol{z}^k-\mu(\bm A_2 \bm r_1^{k+1} + \bm B_1 \bm \beta_1^k + \bm B_2 \bm \beta_2^k + \bm B_3 \bm b^k - \bm c) \right] \right\}\ge 0,\\%
\theta_3(\boldsymbol r_2 )-\theta_3(\boldsymbol{r}_2^{k+1})+(\boldsymbol r_2-\boldsymbol{r}_2^{k+1})^\top \left\{-\boldsymbol{A_3}^\top \left[ \boldsymbol{z}^k-\mu(\bm A_3 \bm r_2^{k+1} + \bm B_1 \bm \beta_1^k + \bm B_2 \bm \beta_2^k + \bm B_3 \bm b^k - \bm c) \right] \right\}\ge 0,\\%
\theta_4(\boldsymbol \beta_1 )-\theta_4(\boldsymbol{\beta}_1^{k+1})+(\boldsymbol \beta_1-\boldsymbol{\beta}_1^{k+1})^\top \left\{-\boldsymbol{B_1}^\top \left[ \boldsymbol{z}^k-\mu(\bm A_1 \bm \beta^{k+1} + \bm A_2 \bm r_1^{k+1} +  \bm A_3 \bm r_2^{k+1} + \bm B_1 \bm \beta_1^{k+1} - \bm c) \right] \right\}\ge 0, \\%
\theta_5(\boldsymbol \beta_2 )-\theta_5(\boldsymbol{\beta}_2^{k+1})+(\boldsymbol \beta_2-\boldsymbol{\beta}_2^{k+1})^\top \left\{-\boldsymbol{B_2}^\top \left[ \boldsymbol{z}^k-\mu(\bm A_1 \bm \beta^{k+1} + \bm A_2 \bm r_1^{k+1} +  \bm A_3 \bm r_2^{k+1}  + \bm B_2 \bm \beta_2^{k+1}  - \bm c) \right] \right\}\ge 0, \\%
\theta_6(\boldsymbol b )-\theta_6(\boldsymbol{b}^{k+1})+(\boldsymbol b-\boldsymbol{b}_2^{k+1})^\top \left\{-\boldsymbol{B_3}^\top \left[ \boldsymbol{z}^k-\mu(\bm A_1 \bm \beta^{k+1} + \bm A_2 \bm r_1^{k+1} +  \bm A_3 \bm r_2^{k+1}  + \bm B_3 \bm b^{k+1} - \bm c) \right] \right\}\ge 0,
\end{array}\right.
\end{equation}
\end{footnotesize}
which is also the first-order optimality condition of the scheme
\begin{small}
\begin{equation}\label{4priupdate}
\left\{
\begin{array}{l}
(\boldsymbol{\beta}^{k+1}, \bm r_1^{k+1}, \bm r_2^{k+1})=\mathop{\arg \min }\limits_{\boldsymbol \beta, \bm r_1, \bm r_2} \left\{ \theta_1(\boldsymbol{\beta}) + \theta_2(\bm r_1) + \theta_3(\bm r_2) - (\boldsymbol{z}^{k})^\top (\boldsymbol{A_1}\boldsymbol{\beta} + \bm A_2 \bm r_1 + \bm A_3 \bm r_2) \right. \\
\left. + \frac{\mu}{2}\| \bm A_1 \bm \beta + \bm A_2 \bm r_1 +  \bm A_3 \bm r_2 + \bm B_1 \bm \beta_1^k + \bm B_2 \bm \beta_2^k + \bm B_3 \bm b^k - \bm c\|_2^2 + \frac{1}{2}\|\bm \beta - \bm \beta^k\|_{ \bm S_{\bm G}}^{2} \right\}, \\%
(\boldsymbol{\beta}_1^{k+1}, \boldsymbol{\beta}_2^{k+1}, \boldsymbol{b}^{k+1})= \mathop{\arg \min }\limits_{\boldsymbol{\beta}_1,\boldsymbol{\beta}_2,\boldsymbol{b}} \left\{\theta_4(\boldsymbol \beta_1)+\theta_5(\boldsymbol \beta_2)+\theta_6(\boldsymbol b)- (\boldsymbol{z}^{k})^\top(\boldsymbol{A_4}\boldsymbol{\beta}_1+\boldsymbol{A_5}\boldsymbol{\beta}_2+\boldsymbol{A_6}\boldsymbol{b}) \right. \\
\left.+ \frac{\mu}{2}\| \bm A_1 \bm \beta^{k+1} + \bm A_2 \bm r_1^{k+1} +  \bm A_3 \bm r_2^{k+1} + \bm B_1 \bm \beta_1 + \bm B_2 \bm \beta_2 + \bm B_3 \bm b - \bm c\|_2^2 \right\}.
\end{array}\right.
\end{equation}
\end{small}
Then, we can consider $(\bm\beta^\top, \bm r_1^\top, \bm r_2^\top)^\top$ as a variable $\bm v_1$, $\theta_1(\bm\beta) + \theta_2(\bm r_1) + \theta_3(\bm r_2)$ as a function ${\bm\theta}_1(\bm v_1)$, $[\bm A_1, \bm A_2, \bm A_3]$ as a matrix $\bm A$; and $(\bm \beta_1^\top, \bm \beta_2^\top, \bm b^\top)^\top$ as the other variable $\bm v_2$, $\theta_4(\bm \beta_1) + \theta_5(\bm \beta_2) + \theta_6(\bm b)$ as the other function ${\bm\theta}_2(\bm v_2)$, $[\bm B_1, \bm B_2, \bm B_3]$ as the other matrix $\bm B$. Thus,  the update of the primal variables in (\ref{4priupdate}) can be rewritten as
\begin{small}
\begin{equation}\label{2-block}
\left\{
\begin{array}{l}
\bm v_1^{k+1} = \mathop{\arg \min }\limits_{\bm v_1} \left\{{\bm\theta}_1(\bm v_1) - (\bm z^k)^\top \bm A \bm v_1 + \frac{\mu}{2} \|\bm A \bm v_1 + \bm B \bm v_2^k - \bm c  \|_2^2 + \frac{1}{2}\|\bm \beta - \bm \beta^k\|_{ \bm S_{\bm G}}^{2}  \right\}, \\
\bm v_2^{k+1} = \mathop{\arg \min }\limits_{\bm v_2} \left\{{\bm\theta}_2(\bm v_2) - (\bm z^k)^\top \bm B \bm v_2 + \frac{\mu}{2} \|\bm A \bm v_1^{k+1} + \bm B \bm v_2  - \bm c \|_2^2  \right\}.
\end{array}\right.
\end{equation}
\end{small}
Obviously, if we incorporate the update of the dual variable $\bm z^{k+1} = \bm z^k - \mu ( \bm A \bm v_1^{k+1} + \bm B \bm v_2^{k+1}-  \bm c) $ into (\ref{2-block}), it becomes a two-block ADMM iteration process. Although  the iteration in (\ref{2-block}) can be seen as a two-block ADMM algorithm, it only linearizes a portion of $\bm v_1$, namely $\bm \beta$. 
Note that when $M>2$, we can still transform Algorithm \ref{alg1} into a partially linearized ADMM algorithm by defining $\bm v_1 = (\bm{\beta}^\top, \bm{r}_1^\top, \bm{r}_2^\top, \dots, \bm{r}_M^\top)^\top$ and $\bm v_2 = (\bm{\beta}_1^\top, \bm{\beta}_2^\top, \dots, \bm{\beta}_M^\top, \bm{b}^\top)^\top$, and making some corresponding changes to $\bm\theta_1(\bm v_1)$, $\bm\theta_2(\bm v_2)$, $\bm A$ and $\bm B$. We include the mathematical derivation regarding this result in Appendix \ref{B}.
 Therefore, the existing convergence results in \cite{BY,BY2} and \cite{Y2020}  cannot be directly applied to Algorithm \ref{alg1}.

Nevertheless, we can also follow the similar framework in \cite{BY,BY2} and \cite{Y2020} for convergence analysis. The convergence property and  convergence rate of Algorithm \ref{alg1} are shown in the following theorem.  We  provide a detailed proof in the Appendix  \ref{B}.  
\begin{thm}\label{TH2}
Let $\tilde{\bm w}^k =  \left(\boldsymbol{\beta}^k, \{\boldsymbol{r}_m^k\}_{m=1}^{M},  \{\boldsymbol{\beta}_m^k\}_{m=1}^{M}, \boldsymbol{b}^k, \{\boldsymbol{d}_m^k\}_{m=1}^{M}, \{\bm{e}_m^k \}_{m=1}^{M},\boldsymbol{f}^k\right)$ is generated by Algorithm \ref{alg1} with an initial feasible solution $\tilde{\bm w}^0$. The sequence $\tilde{\bm v}^k =  \left(\bm \beta^k, \{ \boldsymbol{\beta}_m^k\}_{m=1}^{M},  \bm b^k,
\{\boldsymbol{d}_m^k\}_{m=1}^{M}, \{\bm{e}_m^k \}_{m=1}^{M},\boldsymbol{f}^k \right)$ converges to 
$\tilde{\boldsymbol{v}}^{*}$, where  $\tilde{\boldsymbol{v}}^{*}$ $=  \left(\bm \beta^*, \{ \boldsymbol{\beta}_m^*\}_{m=1}^{M}, \bm b^*, \\
\{\boldsymbol{d}_m^*\}_{m=1}^{M}, \{\bm{e}_m^* \}_{m=1}^{M},\boldsymbol{f}^* \right)$ is an optimal solution point of (\ref{proofg1}). The $O(1/k)$ convergence rate in a non-ergodic sense can also be obtained, i.e.,
\begin{equation}\label{th1e}
\| \tilde{\boldsymbol{v}}^{k}-\tilde{\boldsymbol{v}}^{k+1} \|_{\boldsymbol{H}}^{2} \le \small{\frac{1}{k+1}}\| \tilde{\boldsymbol{v}}^{0}-\tilde{\boldsymbol{v}}^{*} \|_{\boldsymbol{H}}^{2},
\end{equation}
where  $\boldsymbol{H} = \begin{bmatrix}
\bm S_{\bm G}   & \bm 0 & \bm 0 \\
\bm 0 &\mu \bm B^\top \bm B & \bm 0 \\
\bm 0 &\bm 0 & \frac{1}{\mu} \bm I_{n_{\bm G} +(M+1)p}
\end{bmatrix} $ is a symmetric and positive semidefinite matrix, and $n_{\bm G}= n$ if $\bm G = \bm I_p$, and otherwise $n_{\bm G}= n-1$. 
\end{thm}
\begin{rem}
Note that  when $\bm G = \bm  I_p$, $\bm S_{\bm G} = \bm 0$. The convergent sequence is  $\left( \{ \boldsymbol{\beta}_m\}_{m=1}^{M}, \bm b, 
\{\boldsymbol{d}_m\}_{m=1}^{M}, \{\bm{e}_m \}_{m=1}^{M},\boldsymbol{f} \right).$ 
Recalling the $\boldsymbol{\beta}$-subproblem and $\bm r_m$-subproblem in Algorithm \ref{alg1}, it becomes evident that as the sequence$\left( \{ \boldsymbol{\beta}_m^k\}_{m=1}^{M}, \bm b^k, \right.$
$\left. \{\boldsymbol{d}_m^k\}_{m=1}^{M}, \{\bm{e}^k_m \}_{m=1}^{M},\boldsymbol{f}^k \right)$ converges, both $\boldsymbol{\beta}^k$ and $\bm r_m^k$ converge to their respective optimal values $\boldsymbol{\beta}^*$ and $\boldsymbol{r}_m^*$.  Note also that the convergence of our algorithm does not require specifying the form of the loss function, which opens up possibilities for further extensions of our algorithm.  
\end{rem}
\begin{rem}
Note that  when $\bm G = \bm F$, $\bm S_{\bm G} = \bm S$. The convergent sequence is  $\tilde{\bm v} = \left( \bm \beta, \{ \boldsymbol{\beta}_m\}_{m=1}^{M}, \bm b, 
\{\boldsymbol{d}_m\}_{m=1}^{M}, \{\bm{e}_m \}_{m=1}^{M},\boldsymbol{f} \right).$
Recalling the $\bm r_m$-subproblem in Algorithm \ref{alg1}, it becomes evident that as the sequence $\tilde{\bm v}^k$ converges, $\bm r_m^k$ converges to the respective optimal value $\boldsymbol{r}_m^*$.   Moreover, we have $\| \tilde{\boldsymbol{v}}^{0}-\tilde{\boldsymbol{v}}^{*} \|_{\boldsymbol{H}_2}^{2}$ of the order $O(1)$, and as a consequence, $\| \tilde{\boldsymbol{v}}^{k}-\tilde{\boldsymbol{v}}^{k+1} \|_{\boldsymbol{H}_2}^{2}=O(1/k)$.  This means that Algorithm \ref{alg1} has a linear convergence rate.
\end{rem}

\subsection{Analysis of Algorithm Complexity}
 We will first discuss the computational cost that the central machine needs to carry. Reviewing Algorithm \ref{alg1}, we know that the central machine only needs to update $\bm \beta^{k+1}$, $\bm b^{k+1}$, and $\bm f^{k+1}$  in each ADMM iteration. 
From the update of $\bm \beta^{k+1}$ in Section \ref{sec41}, we can observe that the three convex combined regularizations are transformed into the soft-thresholding operator in (\ref{lasso}). The computational cost of computing the soft-thresholding in (\ref{enbeta}), (\ref{grbeta}) and (\ref{fubeta}) is $O(p)$ flops. Note that  (\ref{enbeta}) and (\ref{grbeta}) are the same, and the overall cost of forming $\left(\frac{M(\bar{\bm \beta}^k - \bar{\bm e}^k/ \mu) + (\bm b^k + \bm f^k/\mu)}{M+1}\right)$ is $O(Mp)$ flops, where $\bar{\bm \beta}^k= M^{-1}  \sum_{m=1}^{M} \bm \beta_m^k$ and $\bar{\bm e}^k= M^{-1}  \sum_{m=1}^{M} \bm e_m^k$. Then,  the computational cost of  updating $\bm \beta^{k+1}$ for elastic-net and sparse group regression is $O(Mp + p)= O(Mp)$. The soft-thresholding operator in  (\ref{fubeta})  need to compute $\left(  \bm \beta^k - {\mu}/{\eta} \left[ (M \bm I_p + \bm F^\top \bm F)\bm \beta^k - M(\bar{\bm \beta}^k - {\bar{\bm e}^k}/{\mu})  - \bm F^\top (\bm b^k + {\bm f^k}/{\mu}  )    \right]       \right)$.  In fact, the matrix $\bm F^\top \bm F$ does not need to be calculated, and we can easily write out each element of its matrix due to its own structure. Then,  the main computational cost in generating vector  $\left(  \bm \beta^k - {\mu}/{\eta} \left[ (M \bm I_p + \bm F^\top \bm F)\bm \beta^k - M(\bar{\bm \beta}^k - {\bar{\bm e}^k}/{\mu})  - \bm F^\top (\bm b^k + {\bm f^k}/{\mu}  )    \right]       \right)$ lies in  the matrix multiplied by the vector inside it, which has a complexity of $O(p^2)$. Then,  the computational cost of  updating $\bm \beta^{k+1}$ in  (\ref{fubeta}) for sparse fused regression is $O(Mp + p^2 + p )= O(p^2)$ $(p>>M)$ flops.

For the central machine in each iteration of the ADMM algorithm, the computation required for $\bm b^{k+1}$  in Section \ref{sec41} consists of a linear algebra operation for the elastic net ($O(p)$ flops), a group soft-thresholding operator for the sparse group lasso ($O(p)$ flops), and a soft-thresholding operator for the sparse fused lasso ($O(p^2-p)= O(p^2)$ flops).  The cost of updating the dual variable $\bm f^{k+1}$ on the central machine is $O(np)$ (elastic net and sparse group lasso) or $O(p^2)$ (sparse fused lasso) flops. Therefore, the total computational cost for the central machine in the elastic net and sparse group regression is $O(Mp)+ O(p) +   O(p) = O(Mp) $ flops, while the total computational cost for the central machine in the sparse fused lasso is $O(p^2)+ O(Mp) + O(p^2)  = O(p^2)$ flops.

Next, let us discuss the computational cost of each iteration when implementing Algorithm \ref{alg1} on each local machine. Based on the discussion in Section \ref{sec3}, we can observe that the operations for each local machine remain the same for implementing the three convex combined regularization regressions. 
To update $\bm r^{k+1}$ in (\ref{rm}), we use the $\mathcal L$ proximal operator. The cost of obtaining the closed-form solution of the $\mathcal L$ proximal operator is $O(n_m)$ flops. Then, the cost of updating $\bm r^{k+1}$ depends on the matrix-vector multiplication ($O(n_mp)$) within  it. Therefore, the computational cost of updating $\bm r^{k+1}$ is $O(n_m+n_mp)= O(n_mp)$ flops.

Before each iteration of the ADMM algorithm, each local machine needs to compute the inverse of the matrix $\bm X_m^\top \bm X_m + \bm I_p$. The computational cost is $O(n_mp^2)$ flops if $n_m > p$. However, when $p>n_m$, we can use the Woodbury matrix identity to compute the inverse, which reduces the cost to $O(n_m^2p)$. Once the inverse of the matrix is obtained, the cost of obtaining $\bm \beta_m^{k+1}$ in (\ref{betam}) is $\max\left(O(n_mp),O(p^2)\right)$ flops, which depends on the matrix-vector multiplication involved in the calculation. The cost of updating the dual variable $\bm d_m^{k+1}$ and $\bm e_m^{k+1}$ on each local machines is $O(n_mp)$ and $O(p)$  flops. Thus,  the total computational cost for each local machine is 
\begin{equation}
\begin{cases}
  O(n_mp^2) +  O(n_mp) \times K  & \text{If } n_m > p,\\
O(n_m^2p) + O(p^2) \times K  & \text{otherwise},
\end{cases}
\end{equation}
where  $K$  is the total number of iterations in the ADMM algorithm.

In the practical implementation of the algorithm, it is common to have $n_m >> M$ and $p >> M$. As a consequence, the primary factor impacting the overall cost of the algorithm is not the expense of the central machine but rather the cost associated with operating the local machine, which handles the largest quantity of samples.  Let $n_{\max} = \max\{n_m \}_{m=1}^{M}$, then the total complexity of Algorithm \ref{alg1} (in parallel)  for solving the three types of convex combined regularized regressions is given by
 \begin{equation}\label{comp}
\begin{cases}
  O(n_{\max}p^2) +  O(n_{\max}p) \times K  & \text{If } n_{\max} > p,\\
O(n_{\max}^2p) + O(p^2) \times K  & \text{otherwise}.
\end{cases}
\end{equation}
The complexity of this algorithm indicates that the computational cost of handling the three different combined regularization regressions is similar. As it is well-known, the computation of sparse fused lasso is significantly more complex compared to the other two combined regularizations. This is in fact one of the main advantages of our parallel algorithm. 

Furthermore, this algorithm has a comparable complexity to the ADMM algorithm \cite{W2023} when $M=1$, and comparable complexity to the ADMM algorithms proposed in \cite{YLW} and \cite{FLY} when $M\ge2$.

\section{Numerical results}\label{sec6}
In this section, we use synthetic  data and real data cases to demonstrate the accuracy, stability, and scalability of the proposed ADMM algorithms.
\subsection{Synthetic data}\label{sec61}
In this subsection,  we synthetically generate regression datasets that are suitable for three combinations of regularization, and solve them using the ADMM  algorithms that we propose in both parallel computing framework and non-parallel computing framework. In additions, we will compare the performance of our algorithm with other state-of-the-art methods in terms of estimation accuracy and computational time.
%%%%%%
\subsubsection{Elastic-net}\label{sec611}
In the first study, we utilize a well-known simulation model proposed by \cite{FHT} and \cite{GFK} to generate data for conducting a comparison of elastic-net regression.  We simulate data with $n$ observations from the linear model
\begin{align}\label{lmodel}
\bm y = \sum_{j=1}^{p} \bm x_j \beta_j^* + \sigma \bm  \epsilon,
\end{align}
where $(\bm x_1, \bm x_2, \dots, \bm x_p)^\top \sim N(\bm 0_p,\bm \Sigma)$,$\bm \Sigma = (\rho + (1-\rho)\bm I_{i=j})_{p \times p} $; $\beta_j^* = (-1)^j \exp^{-(2j-1)/20} $;  $\bm \epsilon \sim N(0,1) $; and the parameter $\sigma$ is selected in such a way that the signal-to-noise ratio of the data is equal to 1. For our simulation, we focus on two scenarios: high-dimensional data on a single machine (where $n=720, p=2560$), and massive data on multiple machines (where $n=720000,p=1280$) with the choice of the correlation $\rho = 0.5$. The schematic diagram of the true coefficient is shown in Figure \ref{Fig4}.
\begin{figure}[H]
\centering
\includegraphics[width=7cm,height=4cm]{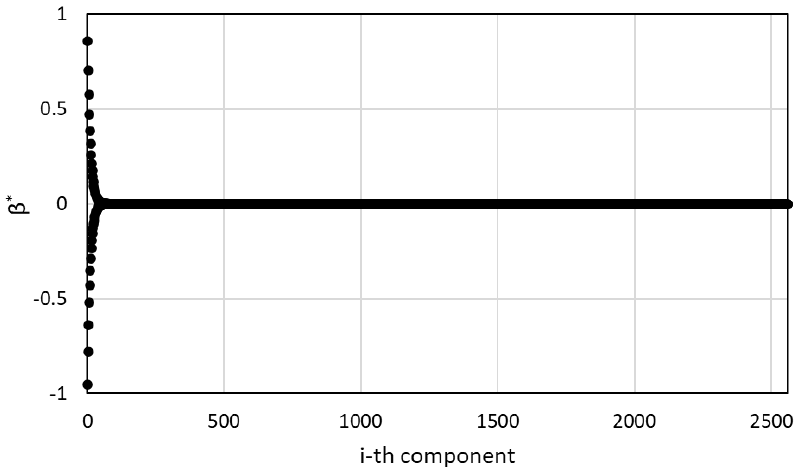}
\caption{\footnotesize{The true coefficient for elastic-net regressions.}}\label{Fig4}
\end{figure}
First, we apply our CPADMM algorithm to solve four elastic-net regressions: LS-Enet (least squares loss plus elastic-net), Q-Enet (quantile loss plus elastic-net), SR-Enet (square root loss plus elastic-net), and H-Enet (huber loss plus elastic-net). 
Figure \ref{Fig5} displays the visualization of the coefficient estimations $\{\beta_j^*\}_{j=1}^{p}$ obtained from these elastic-net regression models. 
This result is an arbitrary selection made by us, indicating that the four elastic-net regression models all perform well under normal error, and  CPADMM is an effective and accurate algorithm for solving them.
\begin{figure}[H]
\centering
\includegraphics[width=12cm,height=7cm]{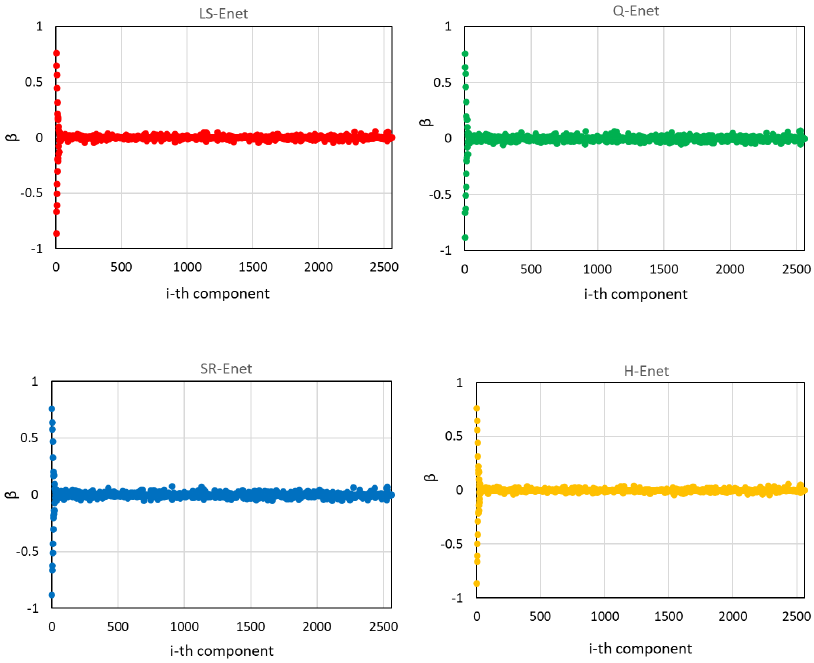}
\caption{\footnotesize{The estimation coefficients by CPADMM for the four types of elastic-net regression.}}\label{Fig5}
\end{figure}

Next, we  compare the CPADMM algorithm in this paper with some state-of-the-art algorithms to fit  the regression data in (\ref{lmodel}).  The above four elastic-net regression models all have efficient solvers, such as glmnet in \cite{FHT} for LS-Enet, FHDQR in \cite{GFK} for Q-Enet, cyclic coordinate-wise
descent (CCD in \cite{R2017}) for SR-Enet, and hqreg in \cite{Y2017} for H-Enet. We compare our unified ADMM algorithm with the above four algorithms for solving elastic-net regression models. Each simulation is repeated 100 times, and the average results for nonparallel computations (single machine) are presented in Table \ref{Tab3}. The results of Table \ref{Tab3} indicate that although our algorithm has certain disadvantages compared to existing algorithms in terms of time, it has certain advantages in estimation accuracy and prediction performance.

Sometimes it is necessary to set $M > 1$ when the data is too large to be stored  processed by a single computer, or can only be stored in a distributed manner.
Due to the lack of research on parallel algorithms for massive data regression with  combined regularizations at present, this paper only demonstrates the performance of our algorithm in a parallel computing environment.
We present the computation time and coefficient estimation accuracy for $M = 1, 10$, and $100$ in Figure \ref{Fig3}. The results in Figure \ref{Fig3} indicate that parallel algorithms can effectively save computation time on large-scale data, but excessive local machines can also slow down the convergence speed of the algorithm.

\begin{table}\small
\centering
\caption{\footnotesize{Comparison of CPADMM and four algorithms for calculating  four elastic-net regressions}}
\resizebox{\textwidth}{!}{
\begin{threeparttable}
\begin{tabular}{lcccccccc}
\Xhline{1pt}
 & \multicolumn{2}{c}{LS-Enet} & \multicolumn{2}{c}{Q-Enet} & \multicolumn{2}{c}{SR-Enet} & \multicolumn{2}{c}{H-Enet} \\
\Xhline{0.4pt}
Algoritnm & glmnet & CPADMM & FHDQR & CPADMM & CCD & CPADMM & hqreg & CPADMM \\
AAE($\times 10^{-2}$) & 0.284(0.011) & \bf 0.251(0.010) & 0.415(0.019) & \bf 0.379(0.016) & 0.467(0.023) & \bf 0.389(0.015) & 0.385(0.017) & \bf 0.282(0.012) \\ 
ASE($\times 10^{-2}$) & 0.875(0.066) & \bf0.753(0.059) & 0.999(0.072) & \bf 0.951(0.065) & 1.236(0.064) & \bf 0.962(0.058) & 1.010(0.061) & \bf 0.843(0.052) \\ 
AAP & 0.155(0.012) & \bf 0.131(0.008) & 0.022(0.003) & \bf 0.009(0.001) & 0.024(0.005) & \bf 0.005(0.001) & \bf 0.047(0.009) & 0.133(0.015) \\ 
ASP & 0.183(0.013) &\bf 0.173(0.011) & 0.104(0.007) & \bf 0.025(0.003) & 0.013(0.002) & \bf 0.007(0.001) & 0.259(0.023) & \bf 0.172(0.018) \\ 
Time &\bf 0.052(0.005) & 8.471(0.076) & \bf 2.765(0.023) & 8.897(0.079) & 10.32(0.121) & \bf 8.834(0.072) & \bf 3.342(0.039) & 8.532(0.080) \\ 
\Xhline{1pt}
\end{tabular}
\begin{tablenotes}
        \footnotesize
        \item[*] The meanings of the notations used in this table are as follows: AAE: average absolute estimation error; ASE: average square estimation error; AAP: average absolute prediction error;  ASP: average square prediction error. Numbers in the parentheses represent the corresponding standard deviations. The optimal solution is represented in bold.
\end{tablenotes}
\end{threeparttable}}
\label{Tab3}
\end{table}
\begin{figure}[H]
\centering
\includegraphics[width=16cm,height=8cm]{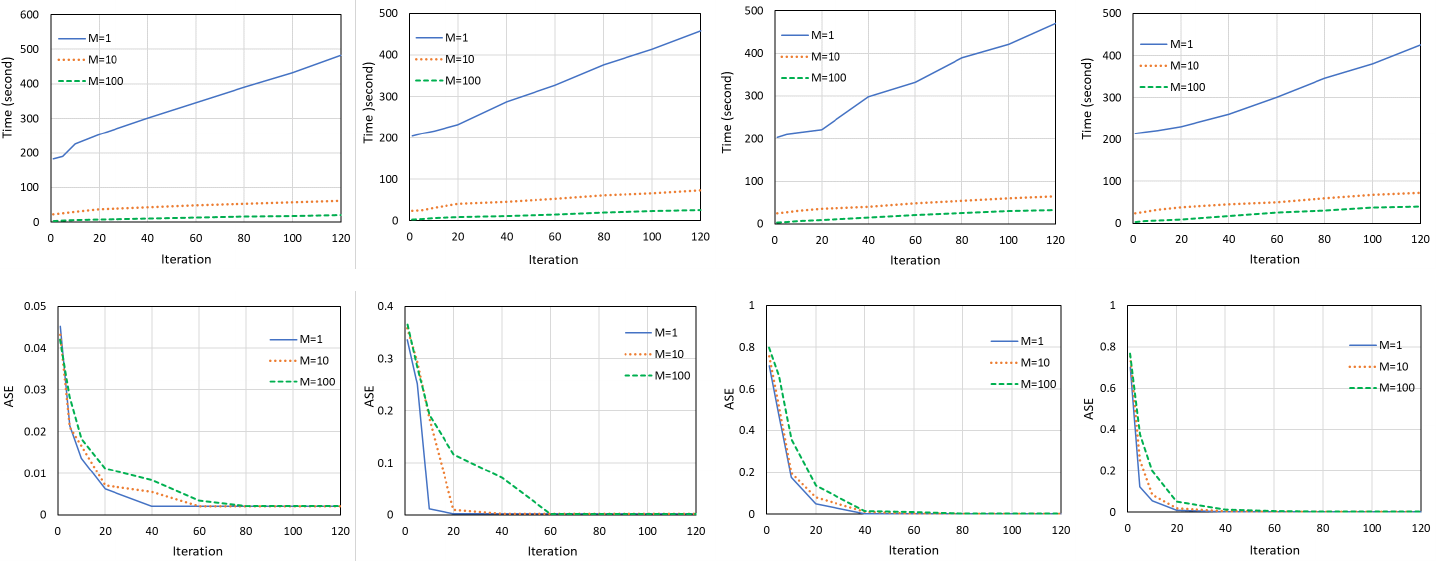}
\caption{\footnotesize{Comparisons of  elastic-net regressions for different $M$ value.}}\label{Fig3}
\end{figure}

%%%%%
\subsubsection{Sparse fused lasso}\label{sec612}
In the second study, we do some simulations on fused regressions.   All model settings remain unchanged from those in Section \ref{sec611}, except for the coefficient settings. In this case, the coefficients are assumed to exhibit an ordering relationship and are closely connected to their neighboring coefficients. To this end, we divide the coefficients into 80 groups, and randomly select 10 groups to set as the same non-zero coefficients. For convenience, we define the set of these non-zero groups as $\{\mathcal{A}_i\}_{i=1}^{10}$. 
The coefficients are  randomly  generate  as
\begin{equation}\label{sbeta}
\boldsymbol{\beta} _{\mathcal{A}_i}^* = \left\{ \begin{array}{l}
\text{U}\left[ { - 3,3} \right] \times \boldsymbol{1}_{\mathcal{A}_i},\ \text{if} \ i=1,2,\dots,10,\\
\textbf{0},\ \ \ \ \ \ \ \ \ \ \ \ \ \ \ \ \ \ \  \text{otherwise}.
\end{array} \right.
\end{equation}
Here, $\text{U} \left[ -3,3 \right]$ represents the uniform distribution on the interval $\left[ { - 3,3} \right]$.  The coefficients have been used in \cite{LMY},  \cite{XLLK} and \cite{W2023}, and the schematic diagram of a random result of the coefficients is shown in Figure \ref{Fig10}. 
\begin{figure}[H]
\centering
\includegraphics[width=7cm,height=4cm]{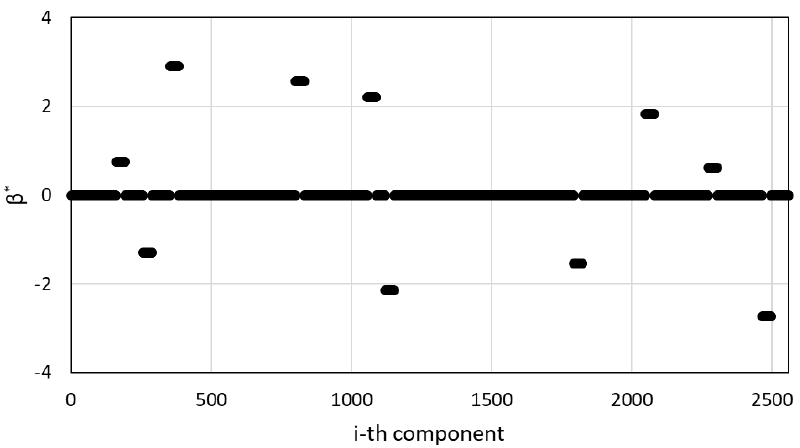}
\caption{\footnotesize{The true coefficient for sparse fused lasso regressions.}}\label{Fig10}
\end{figure}
We also first used our CPADMM algorithm to calculate four sparse fused lasso regression models , that is, LS-Sfla (least squares loss plus  sparse fused lasso), Q-Sfla (quantile loss plus sparse fused lasso), SR-Sfla (square root loss plus  sparse fused lasso), H-Sfla (huber loss plus  sparse fused lasso), and their estimation coefficients as shown in  Figure \ref{Fig9}. The results in Figure \ref{Fig9} demonstrate the good performance of the four sparse fused lasso regression models with normal error. Furthermore, CPADMM is an effective and accurate algorithm for solving these models.
\begin{figure}[H]
\centering
\includegraphics[width=12cm,height=7cm]{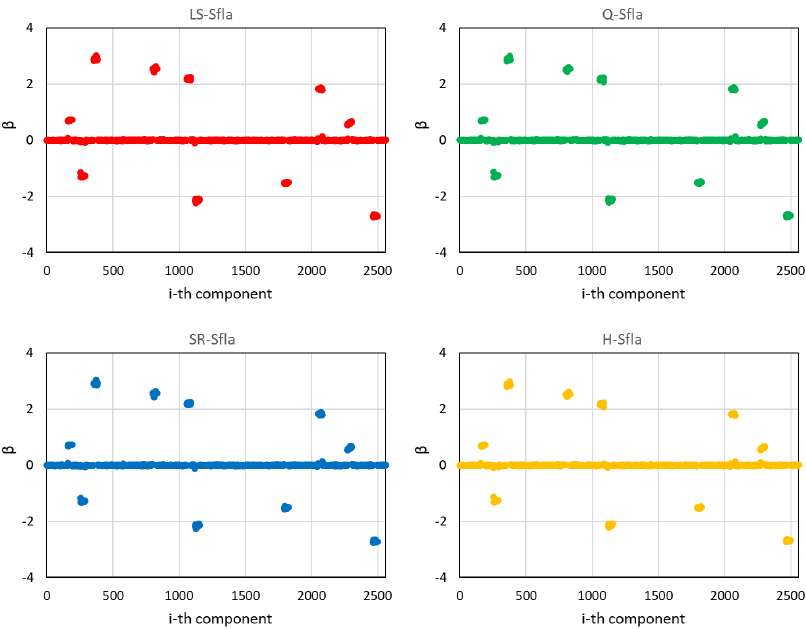}
\caption{\footnotesize{The estimation coefficients by CPADMM for the four types of sparse fused lasso regression.}}\label{Fig9}
\end{figure}

Next, we use four variations of sparse fused lasso regression to fit the data. Several algorithms have been proposed to solve these four sparse fused lasso models. For instance, the linearized ADMM (LADMM) algorithm is used for LS-Sfla in \cite{LMY}, the Multi-block ADMM (MADMM) algorithm is used for Q-Sfla in \cite{W2023}, the square root ADMM (SRADMM) algorithm is used for SR-Sfla in \cite{JLD}, and the path solution algorithm is used for H-Sfla in \cite{LZL2}.
To compare the performance of our unified CPADMM algorithm with the above four algorithms in solving sparse fused lasso regression models, we conduct 100 repetitions of each simulation. The average results for non-parallel computations (single machine) are presented in Table \ref{Tab4}. The results in Table \ref{Tab4} indicate that our algorithm demonstrates competitive performance in terms of computation time, estimation accuracy, and prediction performance compared to existing algorithms.

We can implement parallel simulations using our CPADMM algorithm to conclude this subsection.
We show the computation time and coefficient estimation accuracy for $M = 1, 10$, and $100$ in Figure \ref{Fig2}. The results demonstrate that parallel algorithms can significantly reduce computation time for large-scale sparse fused lasso regression. However, it should be noted that an excessive number of local machines can potentially impact the convergence speed of the algorithm.

\begin{figure}[H]
\centering
\includegraphics[width=16cm,height=8cm]{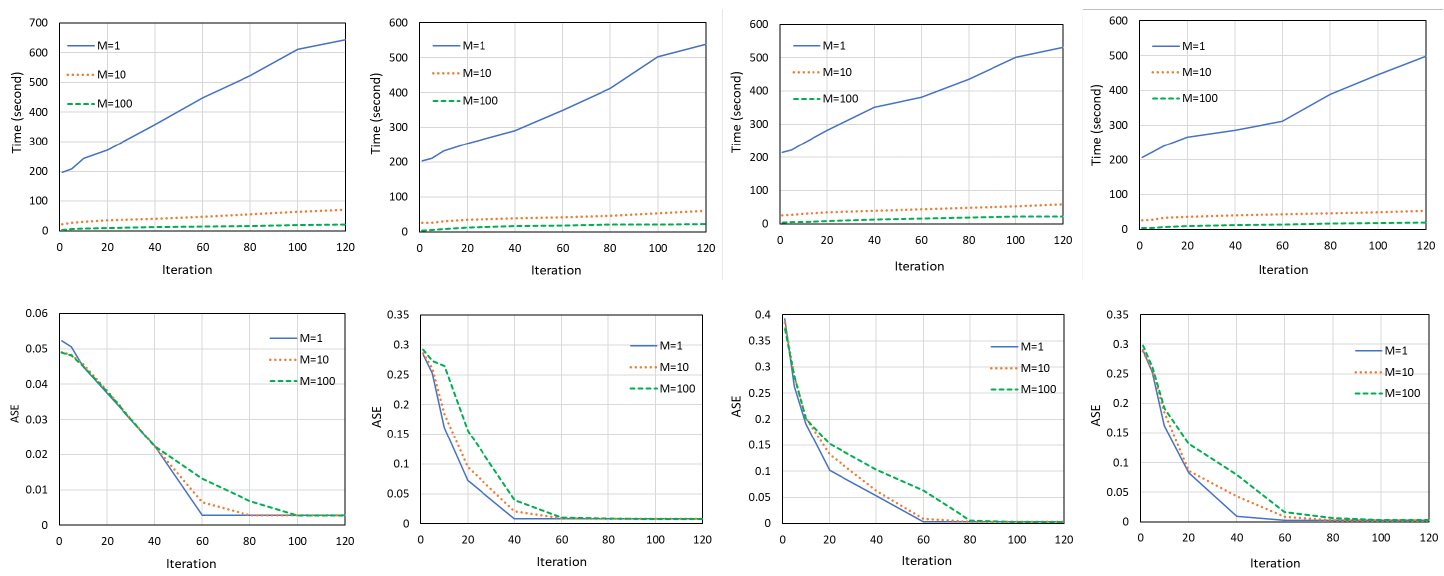}
\caption{\footnotesize{Comparisons of  sparse fused lasso regressions for different $M$ value.}}\label{Fig2}
\end{figure}

\begin{table}\small
\centering
\caption{\footnotesize{Comparison of CPADMM and four algorithms for calculating  four sparse fused lasso regressions.}}
\resizebox{\textwidth}{!}{
\begin{threeparttable}
\begin{tabular}{lcccccccc}
\Xhline{1pt}
 & \multicolumn{2}{c}{LS-Sfla} & \multicolumn{2}{c}{Q-Sfla} & \multicolumn{2}{c}{SR-Sfla} & \multicolumn{2}{c}{H-Sfla} \\
\Xhline{0.4pt}
Algorithm & LADMM & CPADMM & MADMM & CPADMM & SRADMM & CPADMM & Path Solution & CPADMM\\
AAE($\times 10^{-2}$) & 0.017(0.008) & \bf 0.008(0.001) & 0.010(0.004) & \bf 0.009(0.004) & 0.022(0.029) & \bf 0.008(0.001) & 0.015(0.014) & \bf 0.008(0.001) \\ 
ASE($\times 10^{-2}$) & 0.026(0.013) & \bf 0.015(0.009) & \bf 0.016(0.006) & 0.017(0.005) & 0.033(0.015) & \bf 0.016(0.007) & 0.024(0.011) & \bf 0.015(0.006) \\ 
AAP & 0.621(0.028) & \bf 0.451(0.019) & \bf 0.466(0.017) & 0.468(0.018) & 0.592(0.025) & \bf 0.339(0.013) & 0.483(0.018) & \bf 0.461(0.012) \\ 
ASP & 0.583(0.025) &\bf 0.571(0.018) & \bf 0.657(0.031) & 0.659(0.029) & 0.738(0.039) & \bf 0.427(0.019) & 0.602(0.032) &\bf 0.589(0.025)\\ 
Time & 13.59(0.442) & \bf 12.48(0.437) & 12.62(0.413) & \bf 12.55(0.409) & 15.40(0.620) & \bf 12.22(0.425) & 230.6(23.61) & \bf 12.52(0.418) \\ 
\Xhline{1pt}
\end{tabular}
\begin{tablenotes}
        \footnotesize
        \item[*] The meanings of the notations used in this table are as follows: AAE: average absolute estimation error; ASE: average square estimation error; AAP: average absolute prediction error;  ASP: average square prediction error. Numbers in the parentheses represent the corresponding standard deviations. The optimal solution is represented in bold.
\end{tablenotes}
\end{threeparttable}}
\label{Tab4}
\end{table}

%%%%%
\subsubsection{Sparse group lasso}\label{sec613}
In the third study, we use a simulation model similar to that of \cite{LMY}   to generate data for conducting a comparison of sparse group lasso regression.  All model settings remain the same as those in Section \ref{sec611}, except for the coefficient settings. Taking into account coefficient grouping, we also divide the coefficients into 80 groups. From these groups, we randomly select 10 groups to set as non-zero coefficient groups. For convenience, we define the set of these non-zero coefficients as $S$. Then, the true coefficient vector $\bm \beta^*$ is randomly generated by
\begin{equation}\label{gbeta}
{\beta}^*_j = \left\{ \begin{array}{l}
\xi_j(1 + |a_j|),\ \ \ \ \text{if} \  j  \in S,\\
\text{0},\ \ \ \ \ \ \ \ \ \ \  \ \ \ \ \ \ \text{otherwise},
\end{array} \right.
\end{equation}
where $\xi_j$ is randomly chosen from the set $\{+1,-1\}$ and $a_j$ follows a $N (0, 1)$ distribution. The schematic diagram of a random true coefficient is as follows
\begin{figure}[H]
\centering
\includegraphics[width=7cm,height=4cm]{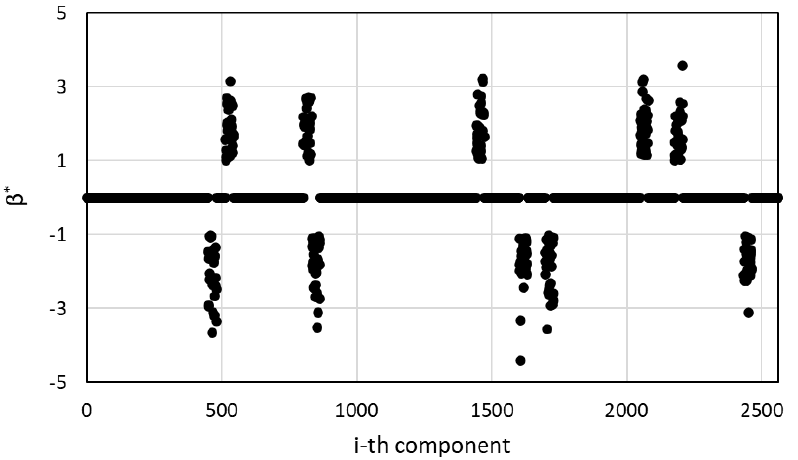}
\caption{\footnotesize{The true coefficient for sparse group lasso regressions.}}\label{Fig8}
\end{figure}
\begin{figure}[H]
\centering
\includegraphics[width=12cm,height=7cm]{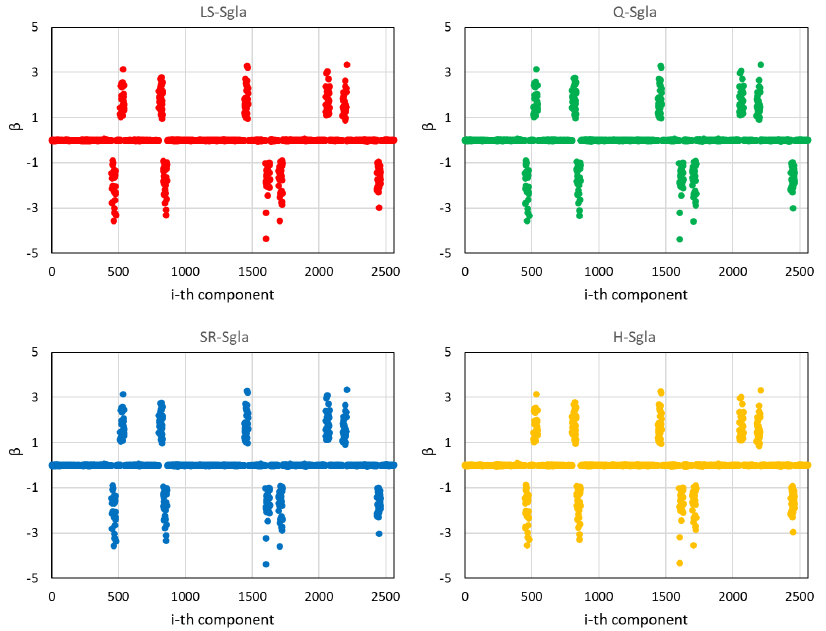}
\caption{\footnotesize{The estimation coefficients by CPADMM for the four types of sparse group lasso regression.}}\label{Fig7}
\end{figure}
We initially employed our CPADMM algorithm to compute four sparse group lasso regression models. These models are LS-Sgla (least squares loss plus sparse group lasso), Q-Sgla (quantile loss plus sparse group lasso), SR-Sgla (square root loss plus sparse group lasso), and H-Sgla (huber loss plus sparse group lasso). The estimated coefficients for these models are displayed in Figure \ref{Fig7}. The results in  Figure \ref{Fig7} show that the four sparse group lasso regression models all perform well under normal error, and  CPADMM is an effective and accurate algorithm for solving them.

Next, we use four types of sparse group lasso regression to fit these data. \cite{LMY} and \cite{Wa2019} respectively proposed linearized ADMM algorithms to solve LS-Sgla and Q-Sgla. Currently, there is no effective algorithm available to solve SR-Sgla ($\lambda_1 \ne 0$ and $\lambda_2 \ne 0$) and H-Sgla ($\lambda_1 \ne 0$ and $\lambda_2 \ne 0$). However, the ADMM algorithm in \cite{SNEBJ} and a scaled thresholding-based iterative selection
procedure (S-TISP) in \cite{B2013} have been proposed to solve SR-Sgla and H-Sgla respectively in the special case when $\lambda_1=0$.
 We compare our unified ADMM algorithm with the above four algorithms for solving sparse group lasso or group lasso regression models. Each simulation is repeated 100 times, and the average results for nonparallel computations (single machine) are presented in Table \ref{Tab5}.  The simulation results demonstrate that our CPADMM algorithm presents significant advantages compared to existing algorithms in terms of computation time, estimation accuracy, and prediction performance.
\begin{table}\small
\centering
\caption{\footnotesize{Comparison of CPADMM and four algorithms for calculating  four sparse group lasso regressions.}}
\resizebox{\textwidth}{!}{
\begin{threeparttable}
\begin{tabular}{lcccccccc}
\Xhline{1pt}
 & \multicolumn{2}{c}{LS-Sgla} & \multicolumn{2}{c}{Q-Sgla} & \multicolumn{2}{c}{SR-Sgla} & \multicolumn{2}{c}{H-Sgla} \\
\Xhline{0.4pt}
Algorithm & LADMM & CPADMM & LADMM & CPADMM & ADMM & CPADMM & S-TISP & CPADMM\\
AAE($10^{-2}$) & 0.023(0.003) & \bf 0.011(0.001) & 0.027(0.004) & \bf 0.012(0.001) & 0.029(0.004) & \bf 0.012(0.001) & 0.036(0.005) & \bf 0.011(0.001) \\ 
ASE($10^{-2}$) & 0.035(0.004) & \bf 0.027(0.002) & 0.042(0.006) & \bf 0.026(0.002) & 0.039(0.005) & \bf 0.025(0.002) & 0.044(0.006) & \bf 0.027(0.002) \\ 
MADE & 0.323(0.013) &\bf  0.219(0.009) & 0.121(0.006) & \bf 0.093(0.003) & 0.088(0.003) & \bf 0.069(0.002) & 0.337(0.015) & \bf 0.219(0.012) \\ 
MSE & 0.381(0.021) & \bf 0.268(0.019) & 0.156(0.013) & \bf 0.129(0.008) & 0.091(0.002) & \bf 0.087(0.001) & 0.358(0.029) & \bf 0.268(0.018) \\ 
Time & 7.21(0.351) & \bf 5.38(0.227) & 7.85(0.412) & \bf 5.56(0.239) & 9.26(0.482) & \bf 5.41(0.253) & 7.77(0.382) & \bf 5.47(0.246) \\ 
\Xhline{1pt}
\end{tabular}
\begin{tablenotes}
        \footnotesize
        \item[*] The meanings of the notations used in this table are as follows: AAE: average absolute estimation error; ASE: average square estimation error; AAP: average absolute prediction error;  ASP: average square prediction error. Numbers in the parentheses represent the corresponding standard deviations. The optimal solution is represented in bold.
\end{tablenotes}
\end{threeparttable}}
\label{Tab5}
\end{table}

In addition, we conclude this subsection by performing parallel simulations on sparse group lasso regressions through our CPADMM algorithm. The computation time and coefficient estimation accuracy are depicted in Figure \ref{Fig6} for $M = 1, 10$, and $100$. The results  suggest that parallel algorithms can significantly reduce computation time for large-scale sparse group lasso regression data. However, it is worth noting that having too many local machines may actually slow down the convergence speed of the algorithm.

\begin{figure}[H]
\centering
\includegraphics[width=16cm,height=8cm]{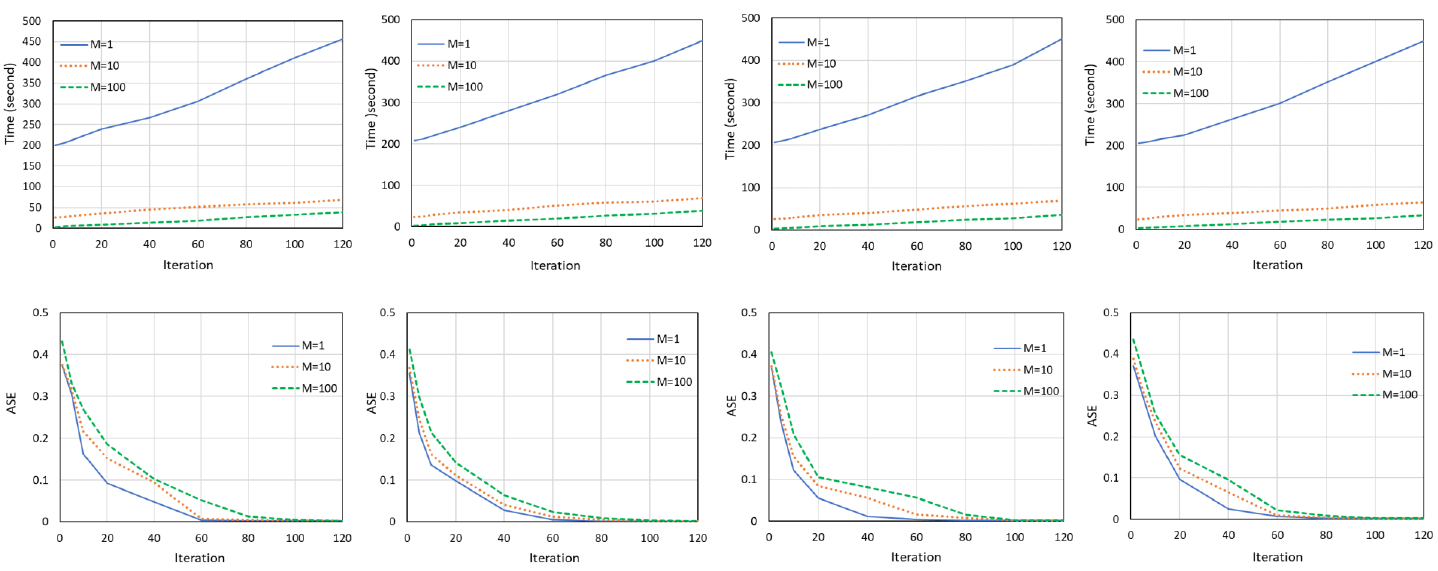}
\caption{\footnotesize{Comparisons of the four types of sparse group lasso regressions for different $M$ value.}}\label{Fig6}
\end{figure}

\subsection{Real datasets}\label{sec62}
In this subsection, we focus on index tracking using combined regularization regression, with the target indices being two highly representative stock indices in China, the SSE 50 (Shanghai Stock Exchange 50 stock index) and CSI 300 (China Securities Index 300).
A stock index is a numerical value calculated by weighting the price or market value of its constituent stocks, which is used to reflect the ups and downs of the stock market and the overall market trend.
SSE 50 is an index of 50 major stocks on the Shanghai Stock Exchange, representing China's A-share market performance and used as a benchmark for investors. CSI 300 tracks the performance of 300 A-share stocks across different sectors on both Shanghai and Shenzhen Stock Exchanges, providing broad coverage of the Chinese stock market.

 For further data analysis, we have tabulated the distribution of constituent stocks in various industry sectors for two stock indices in Table \ref{se}. We classify the constituent stocks of the stock index into several sectors, such as financials (banking, insurance, and financial services); real estate (property development, real estate investment trusts, and property management); consumer discretionary (retail, automotive, and leisure and entertainment); industrials (manufacturing, construction and engineering); utilities (electric utilities and gas utilities); health care (pharmaceuticals and biotechnology); information technology (software development and telecommunications); and others (advertising, agriculture and related fields). The grouping of these sectors can serve as a prior grouping for sparse group lasso regression.
\begin{table}[H]\small
\caption{\footnotesize{Distribution of constituent stocks in various sectors of stock indices.}}
\centering
\label{se}
\begin{tabular}{lll}
\midrule
\text{Sector} & Number of Stocks (SSE 50) &Number of Stocks (CSI 300) \\
\midrule
Financials & 9 & 47 \\
Real Estate & 4 & 20 \\
Consumer Discretionary & 10 & 12 \\
Industrials & 17 & 110\\
Utilities & 3 & 39 \\
Health care & 4 & 25 \\
Information Technology & 3 & 26 \\
Others & 0 & 21 \\
\bottomrule
\end{tabular}
\end{table}

Index tracking refers to investors using financial products such as index funds and exchange-traded funds (ETFs) to replicate or track the performance of specific stock indices. Through index tracking, investors can participate in the ups and downs of the index without directly purchasing individual stocks by holding financial products that track specific indices. These financial products are typically invested in accordance with the weightings of the constituent stocks in the index to ensure their performance closely aligns with the upward or downward trends of the index. The advantage of index tracking is that investors can achieve diversified risk by investing in multiple stocks in the index through a single investment, while also benefiting from the overall market performance. However, selecting which constituent stocks to include in index tracking has become a hot topic in recent years due to the increasing transaction costs associated with holding a larger number of stocks in the financial market.
Since the weights of the constituent stocks in an index are non-negative, regularized regression with non-negativity constraints is a widely used approach, see \cite{W2014}, \cite{YW2016} and \cite{W2022}.

According to their research, we can define $x_{t,j}$ and $y_t$ to represent the returns of the $j$-th constituent stock and the index, respectively, where $j = 1,2,\dots,50 \ (300)$. We can then use a linear regression model to describe the relationship between $x_{t,j}$ and $y_t$ as follows:
\begin{equation}\label{22}
{y_t} = \sum\limits_t {{\beta _j}} {x_{t,j}} + {\varepsilon _t}, \qquad t = 1,2, \cdots T,
\end{equation}
where $\beta_j \ge 0$ represents the weight of the $j$-th chosen stock and $\varepsilon _t$ is the error term. In practical applications, the optimal estimate of $\beta_j$ represents the proportion of each stock. For example, if $\hat{\beta}_1=1, \hat{\beta}_2=2$, then when tracking the stock index, for every unit of the 1st labeled stock held, it would be necessary to hold 2 units of the 2nd labeled stock.
To measure the bias for tracking, the Annual Tracking Error (ATE) is used and defined by 
\begin{equation}\label{23}
TrackingErro{r_{Year}} = \sqrt {252}  \times \sqrt {\frac{{{{\sum {(er{r_t} - mean(er{r_t})} }^2}}}{{T - 1}}},
\end{equation}
where $err_{t}=\hat y_{t}-y_{t}$ and $\hat y_{t}$ is the fitted or predicted value of $y_t$ ,for $t=1,2,\dots,T$.
 Following \cite{W2014}, we can choose some penalty parameters from 0 to a sufficiently large positive number which shrinks all coefficients to 0, and the interval between the two parameters is equal. For the given dataset, we can fit the model using $\lambda_2 = \sqrt{\log p/n}$, and $\lambda_1 = 100\sqrt{\log p/n}$, which is large enough to ensure all coefficients are shrunk to 0. In this case, we choose 1000 penalty parameters ranging from 0 to $100 \sqrt{\log p/n}$. The interval between consecutive parameters is set as $0.1\sqrt{\log p/n}$.

Next, we will compare non-negative lasso regression (nlasso) in \cite{W2014}, non-negative adaptive lasso regression (nalasso) in \cite{YW2016}, non-negative ladlasso regression (nladlasso) in \cite{W2022}, and  non-negative sparse group lasso regression (nsglasso, $G = 8$) considered in this paper for tracking the minute-by-minute stock indexes of the SSE 50 and the CSI 300 between August 7, 2023 and November 9, 2023.
Note that all non negative methods mentioned above can be uniformly solved using our proposed CPADMM algorithm.
 Each stock index has a total of 15000 data, which means $n=15000$.
Please note that the last update of the constituent stocks for SSE 50 and CSI 300 was on June 9, 2023. During the data collection period, there were no trading suspensions for any of the constituent stocks of these two indices. Therefore, no data manipulation is required. In addition, we selected the first 12000 samples as the training set and the last 3000 samples as the training set. We record the ATE ($\times 10^4$) comparison results of various non-negative  regression methods in Table \ref{trackingdata}. In the tracking of the SSE 50 Index, we selected 5, 10, and 20 constituent stocks, while in the CSI 300 Index, we chose 10, 20, and 40 constituent stocks. The results from Table \ref{trackingdata} indicate that the sparse group lasso, when utilizing grouping information, performs significantly better in stock index tracking compared to a single regularizer. Moreover, we have placed the schematic diagrams of the 5 constituent stocks selected by the nsglasso method for tracking SSE 50 and the 10 constituent stocks selected for tracking CSI 300 separately in Figure \ref{Fig12} and Figure \ref{Fig11}.

\begin{table}\small
\centering
\caption{\footnotesize{SSE 50 and CSI 300 index tracking data.}}
\label{trackingdata}
\small
\begin{tabular*}{14.5cm}{cccccccc}
    \hline
    &\multicolumn{3}{c}{SSE 50} &\ \ &\multicolumn{3}{c}{CSI 300}\\
    \cline{2-4}
    \cline{6-8}
    Number &Method   &$ATE_{train}$ &$ATE_{test}$    &Number &Method   &$ATE_{train}$ &$ATE_{test}$\\
    \hline
    5	&nladlasso	&5.411	&2.067	&10	&nladlasso	&6.534	&3.258\\
    \ \	&nalasso	&5.893	&2.191  &\ \	&nalasso	&6.436	&3.658\\
    \ \	&nlasso	&6.093	&2.103  &\ \	&nlasso	&6.852	&3.803\\
    \ \	&nsglasso	&\textbf{5.343}	&\textbf{2.035}  &\ \	&nsglasso	&\textbf{6.298}	&\textbf{3.110}\\
    10	&nladlasso	&5.359	&1.987	&20	&nladlasso	&6.489	&3.165\\
    \ \	&nalasso	&5.762	&2.158  &\ \	&nalasso	&6.412	&3.558\\
    \ \	&nlasso	&5.921	&2.058  &\ \	&nlasso &6.782	&3.724\\
    \ \	&nsglasso	&\textbf{5.226}	&\textbf{1.958}  &\ \	&nsglasso	&\textbf{6.111}	&\textbf{3.087}\\
    20	&nladlasso	&5.309	&1.943	&40	&nladlasso	&6.359	&3.156\\
    \ \	&nalasso	&5.708	&1.996  &\ \	&nalasso	&6.387	&3.470\\
    \ \	&nlasso	&5.856	&2.007  &\ \	&nlasso	&6.626	&3.703\\
    \ \	&nsglasso	&\textbf{5.197}	&\textbf{1.923}  &\ \	&nsglasso	&\textbf{6.091}	&\textbf{2.990}\\
  \hline
\end{tabular*}
\end{table}

\begin{figure}[H]
\centering
\includegraphics[width=16cm,height=8cm]{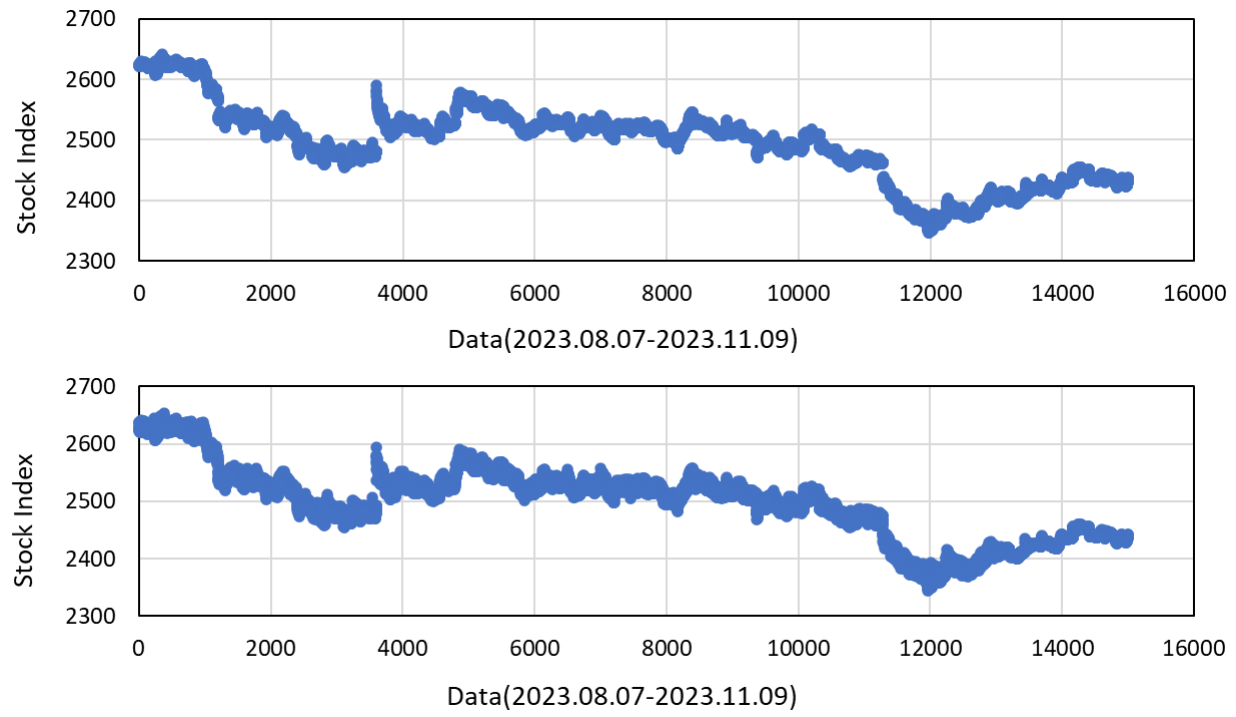}
\caption{\footnotesize{SSE 50 tracking diagram, the top one shows the actual stock index and the bottom one shows the predicted value by nsglasso.}}\label{Fig12}
\end{figure}

\begin{figure}[H]
\centering
\includegraphics[width=16cm,height=8cm]{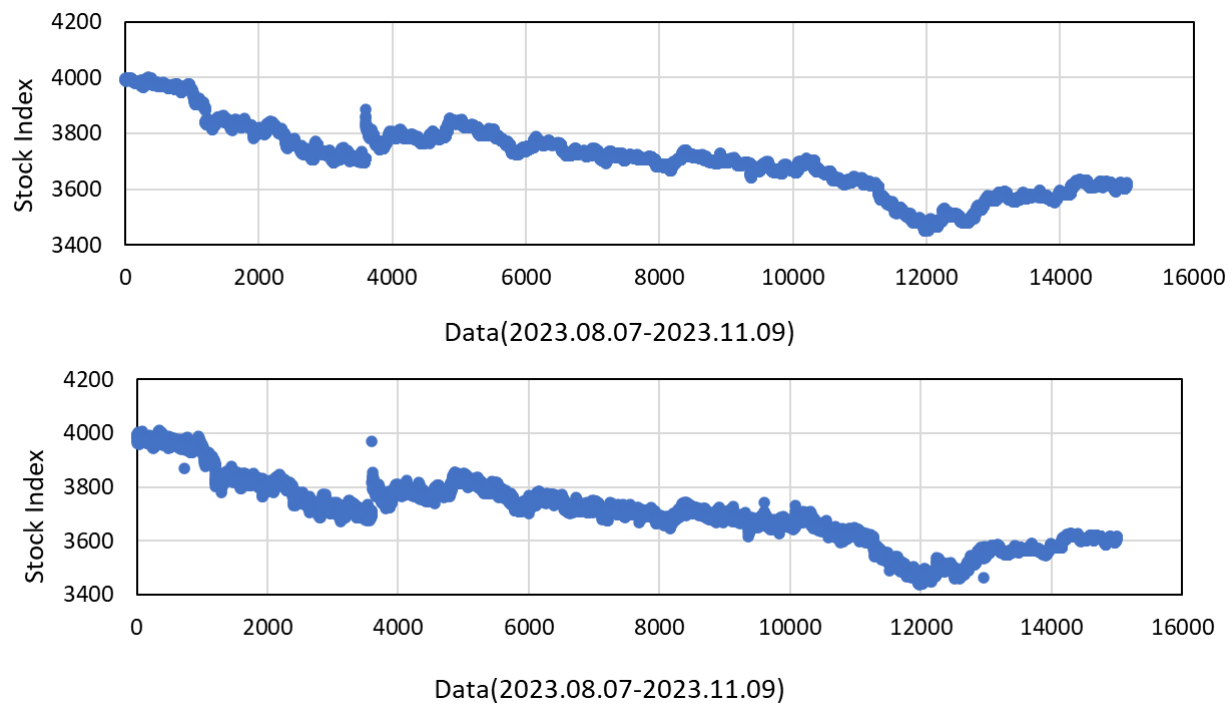}
\caption{\footnotesize{CSI 300 tracking diagram, the top one shows the actual stock index and the bottom one shows the predicted value by nsglasso.}}\label{Fig11}
\end{figure}

\section{Conclusion and further research}\label{sec7}
In this paper, we consider an efficient parallel ADMM algorithm for uniformly solving sparse combined regularized regression models. This work fills the gap in parallel algorithms for combined regularization regression. The primary challenges in designing this unified ADMM algorithm lie in dealing with some combined  regularizations that lack closed-form solutions for the proximal operator and handling distributed storage of data. To solve these two difficulties, we introduce some auxiliary variables that not only enable combined regularization to handle the subproblems of each iteration separately and have closed-form solutions, but also construct consensus structures to adapt to distributed stored data. Fortunately, despite the introduction of additional auxiliary variables, we can prove that these constrained optimization problems can be viewed as optimization problems with two primal variables subject to linear constraints. Since our linearization of the first variable is not complete, but only partial, existing research on the two-block ADMM algorithms cannot directly ensure the convergence of our parallel algorithm. Nevertheless, we can still use the framework used in \cite{BY,BY2} to prove the convergence of our parallel ADMM algorithm and provide its linear convergence rate. The algorithm complexity is comparable to existing ADMM algorithms for both single machine and multi-machine scenarios.

Our algorithm has a relatively fast computation speed due to the closed-form solutions for each subproblem at every iteration. However, we must also acknowledge that our algorithm may be slow when there are too many local machines, particularly when solving for large values of $\{n_m\}_{m=1}^{M}$ and $p$. This is because our algorithm requires the $m$-th local machine to calculate the inverse of a matrix with dimension $\max (n_i, p)$. In addition, as the number of local machines $M$ increases, our algorithm tends to slow down its convergence speed in simulation data.  This is also actually reflected in Theorem \ref{TH2} of this paper, where the convergence rate of the algorithm is inversely proportional to the dimension of the iteration sequence. Developing a more applicable parallel algorithm to address these issues is one of our future research goals.

 Moreover, our algorithm is flexible in terms of the loss function. During the specific implementation of the algorithm, it only requires modifying the proximal operator in  $\bm r$-subproblem to adapt to different losses. This provides the possibility for further extension of our algorithm. For example, similar to \cite{L2015}, we can replace the loss function with the Dantzig selector and $\ell_q$ ($q>1$) to tackle more regression tasks. Additionally, we can utilize certain classification losses mentioned in \cite{L2023} to handle distributed storage data classification tasks. The convergence of these extended  parallel algorithm may be proven by the same method in this paper.

\section*{Acknowledgements}
We express our sincere gratitude to Professor Bingsheng He for engaging in invaluable discussions with us. His insights and expertise have greatly assisted us in effectively utilizing parallel ADMM algorithms to solve regression problems. The research of Zhimin Zhang was supported by the National Natural Science Foundation of China [Grant Numbers 12271066, 12171405, 11871121], and the research of Xiaofei Wu was supported by the Scientific and Technological Research Program of Chongqing Municipal Education Commission [Grant Numbers KJQN202302003].

\begin{footnotesize}

\end{footnotesize}

\newpage
\section*{Appendix}
\appendix

\section{Proofs of Convergence Theorems}\label{B}
In Section \ref{sec41}, we have proven that solving the regression models with combined regularization in parallel using the ADMM algorithm can be seen as a traditional two-block ADMM algorithm ($\bm G = \bm I_p$) or a partially linearized two-block ADMM algorithm ($\bm G = \bm F$). Moreover, the optimization objective function of this paper can be written as a two-block separable function with the linear constraint,
\begin{align}\label{proof1}
& \mathop  {\arg \min }\limits_{\boldsymbol{v}_1, \bm v_2} \big\{{\bm \theta}_1(\bm v_1)  + {\bm \theta}_2(\bm v_2)  \big \},  \notag\\ 
& \textbf{s.t.} \  \bm A \bm v_1 + \bm B \bm v_2= \bm c,
\end{align}
where $\bm v_1 = (\bm \beta^\top, \bm r_1^\top, \bm r_2^\top, \dots, \bm r_M^\top)^\top$,  $\bm v_2 = (\bm \beta_1^\top, \bm \beta_2^\top, \dots, \bm \beta_M^\top, \bm b^\top)^\top$, ${\bm \theta}_1(\bm v_1) = \lambda_1\|\bm\beta \|_1 + I_{\mathcal C}(\bm \beta) + \sum_{m=1}^{M}\mathcal{L}(\bm r_m)$ and ${\bm \theta}_2(\bm v_2) =\lambda_2 \|\bm b \|_{2,1}$ (since the functions corresponding to $\bm \beta_m$ are all zero). Here, $\bm A = [
\bm A_1, \bm A_2, \dots, \bm A_{M+1}]$, where 
$$
{\boldsymbol{A}_1} = \begin{bmatrix}
\boldsymbol G \\
-\boldsymbol{I}_p \\
-\boldsymbol{I}_p \\
\vdots\\
-\boldsymbol{I}_p \\
\bm 0 \\
\bm 0 \\
\vdots \\
\bm 0
\end{bmatrix},
\quad
{\boldsymbol{A}_2} = \begin{bmatrix}
\bm 0 \\
\bm 0 \\
\bm 0 \\
\vdots \\
\bm 0 \\
\bm I_{n_1} \\
\bm 0\\
\vdots \\
\bm 0
\end{bmatrix},
\quad
{\boldsymbol{A}_3} = \begin{bmatrix}
\bm 0 \\
\bm 0\\
\bm 0 \\
\vdots \\
\bm 0 \\
\bm 0 \\
\bm I_{n_2}\\
\vdots \\
\bm 0
\end{bmatrix},
\quad
\dots,
\quad
{\boldsymbol{A}_{M+1}} = \begin{bmatrix}
\bm 0 \\
\bm 0 \\
\bm 0 \\
\vdots \\
\bm 0 \\
\bm 0 \\
\bm 0 \\
\vdots \\
\bm I_{n_M}
\end{bmatrix};
$$
and $\bm B = [\bm B_1, \bm B_2, \dots, \bm B_{M+1}]$, where 
$$
{\boldsymbol{B}_1} = \begin{bmatrix}
\boldsymbol{0} \\
\boldsymbol{I}_p \\
\bm 0 \\
\vdots \\
\bm 0 \\
\bm X_1 \\
\bm 0 \\
\vdots \\
\bm 0
\end{bmatrix},
\quad
{\boldsymbol{B}_2} = \begin{bmatrix}
\bm 0 \\
\bm 0 \\
\bm I_p \\
\vdots \\
\bm 0 \\
\bm 0 \\
\bm X_{2} \\
\vdots \\
\bm 0
\end{bmatrix},
\quad
\dots,
\quad
{\boldsymbol{B}_M} = \begin{bmatrix}
\bm 0 \\
\bm 0 \\
\bm 0 \\
\vdots \\
\bm I_p \\
\bm 0 \\
\bm 0 \\
\vdots \\
\bm X_M
\end{bmatrix},
\quad
{\boldsymbol{B}_{M+1}} = \begin{bmatrix}
- \bm I_G \\
\bm 0 \\
\bm 0 \\
\vdots \\
\bm 0 \\
\bm 0 \\
\bm 0 \\
\vdots \\
\bm 0
\end{bmatrix};$$
and $\bm c = (\bm 0^\top,\bm 0^\top,\bm 0^\top, \dots, \bm 0^\top, \bm y_1^\top, \bm y_2^\top, \dots, \bm y_M^\top)^\top$. Note that all $\bm A_m$ are mutually orthogonal, and $\bm B_m$ are also mutually orthogonal. This is actually the reason for transforming Algorithm \ref{alg1} (multi-block  parallel algorithm) into a two-block ADMM algorithm.

The existing convergence results for the original two-block ADMM, as presented in \cite{BY,BY2} and \cite{Y2020}, also hold for Algorithm \ref{alg1} when $\bm G = \bm I_p$. 
However,  when $\bm G = \bm F$, the existing convergence results for the linearized two-block ADMM in \cite{BY,BY2}  cannot be directly applied to our partially linearized version. Nonetheless, with slight modifications in their proof process, we can also obtain the following conclusions from their existing convergence analysis. 

 Recall that \begin{align}\label{IG2}
\bm S_{\bm G}  =
\begin{cases}
\bm 0, & \text{{if }}  \bm G = \bm I_p,\\
\bm S, & \text{{if }}  \bm G = \bm F, \\
\end{cases}
\end{align}
where $\bm S = \eta \bm I_p - (M \bm I_p + \bm F^\top \bm F)$.
\begin{prop}\label{prop8}
Let $\tilde{\bm w}^k =  \left(\boldsymbol{\beta}^k, \{\boldsymbol{r}_m^k\}_{m=1}^{M},  \{\boldsymbol{\beta}_m^k\}_{m=1}^{M}, \boldsymbol{b}^k, \{\boldsymbol{d}_m^k\}_{m=1}^{M}, \{\bm{e}_m^k \}_{m=1}^{M},\boldsymbol{f}^k\right)$ is generated by Algorithm \ref{alg1} with an initial feasible solution $\tilde{\bm w}^0$.  The sequence $\tilde{\bm v}^k =  \left(\bm \beta^k, \{ \boldsymbol{\beta}_m^k\}_{m=1}^{M},  \bm b^k,
\{\boldsymbol{d}_m^k\}_{m=1}^{M}, \{\bm{e}_m^k \}_{m=1}^{M},\boldsymbol{f}^k \right)$  has
 the following contraction inequality
\begin{equation}\label{iee}
\| \tilde{\boldsymbol{v}}^{k+1}-\tilde{\boldsymbol{v}}^* \|_{\boldsymbol{H}}^{2} \le \| \tilde{\boldsymbol{v}}^{k}-\tilde{\boldsymbol{v}}^* \|_{\boldsymbol{H}}^{2} - \| \tilde{\boldsymbol{v}}^{k}-\tilde{\boldsymbol{v}}^{k+1} \|_{\boldsymbol{H}}^{2},
\end{equation}
where  $\tilde{\boldsymbol{v}}^{*}$ $=  \left(\bm \beta^*, \{ \boldsymbol{\beta}^*\}_{m=1}^{M}, \bm b^*, \\
\{\boldsymbol{d}_m^*\}_{m=1}^{M}, \{\bm{e}_m^* \}_{m=1}^{M},\boldsymbol{f}^* \right)$  and  $\boldsymbol{H} = \begin{bmatrix}
\bm S_{\bm G}   & \bm 0 & \bm 0 \\
\bm 0 &\mu \bm B^\top \bm B & \bm 0 \\
\bm 0 &\bm 0 & \frac{1}{\mu} \bm I_{n_{\bm G} - 1 +(M+1)p}
\end{bmatrix} $ is a symmetric and positive definite matrix, and $n_{\bm G}= n$ if $\bm G = \bm I_p$, and otherwise $n_{\bm G}= n-1$. .
\end{prop}
%%%
\begin{proof}
For convenience,  take $\tilde{\bm r}_M = (\bm r_1^\top, \bm r_2^\top, \dots, \bm r_M^\top)^\top$,  $\tilde{\bm \beta}_M = (\bm \beta_1^\top, \bm \beta_2^\top, \dots, \bm \beta_M^\top)^\top$,  $\bm v_{1,2} = (\bm v_1^\top, \bm v_2^\top)^\top $, $\bm{\theta}(\bm v_{1,2})= \bm \theta_1(\bm v_1)+ \bm \theta_2(\bm v_2)$, and thus $\bm w = (\bm v_{1,2}^\top, \bm z^\top)^\top$ . Because of the orthogonal relationship between the matrices, the subproblems of ADMM in (\ref{proof1}) can be reformulated as
\begin{equation}\label{th4op}
\left\{ \begin{array}{l}
{\boldsymbol{\beta}}^{k+1}= \mathop {\arg \min }\limits_{{\boldsymbol \beta}}\left\{ \lambda \|\bm \beta \|_1+ I_{\mathcal C}(\bm \beta)+{\mu \over 2}\|\boldsymbol{A}_1{\boldsymbol{\beta}}+ \sum_{m=1}^{M}{\boldsymbol{B}_m}\boldsymbol{\beta}_m^k + \bm B_{M+1} \bm b^k-\boldsymbol{c}-\boldsymbol{z}^k/\mu  \|_2^2 + \frac{1}{2}\|\bm \beta - \bm \beta^k\|_{ \bm S_{\bm G}}^{2}  \right\},\\%
{\boldsymbol{r}_1}^{k+1}= \mathop {\arg \min }\limits_{{\boldsymbol r_1}} \left\{\mathcal{L}(\bm r_1) +     {\mu \over 2}\| \bm A_2 \bm r_1  + \bm B_1 \bm \beta_1^k   - \bm c - \bm z^k/\mu     \|_2^2 \right\},\\%
{\boldsymbol{r}_2}^{k+1}= \mathop {\arg \min }\limits_{{\boldsymbol r_2}} \left\{\mathcal{L}(\bm r_2) +     {\mu \over 2}\| \bm A_3 \bm r_2  + \bm B_2 \bm \beta_2^k   - \bm c - \bm z^k/\mu     \|_2^2 \right\},\\%
\qquad  \qquad  \qquad \qquad \qquad \qquad  \vdots
\qquad
\\
{\boldsymbol{r}_M}^{k+1}= \mathop {\arg \min }\limits_{{\boldsymbol r_M}} \left\{\mathcal{L}(\bm r_M) +     {\mu \over 2}\| \bm A_{M+1} \bm r_M  + \bm B_M \bm \beta_M^k   - \bm c - \bm z^k/\mu     \|_2^2 \right\},\\%
\boldsymbol{v}_2^{k+1}=\mathop {\arg \min }\limits_{\boldsymbol v_2}\left\{ \bm{\theta}_2(\boldsymbol{\vartheta})+{\mu \over 2}\|\boldsymbol{A}{\boldsymbol{v}_1}^{k+1}+ {\boldsymbol{B}}\boldsymbol{v}_2-\boldsymbol{c}-\boldsymbol{z}^k/\mu \|_2^2 \right\},\\%
\boldsymbol{z}^{k+1}=\boldsymbol{z}^{k}-\mu(\boldsymbol{A}\boldsymbol{v}_1^{k+1}+ {\boldsymbol{B}}\boldsymbol{v}_2^{k+1} -\boldsymbol c).
\end{array}\right.
\end{equation}
For the ${\boldsymbol \beta}$-subproblem of (\ref{th4op}), it follows from the convexity of $ \lambda \|\bm \beta \|_1$ and $I_{\mathcal C}(\bm \beta)$ that
\begin{small}
\begin{align}\label{ggg1}
\left(\lambda \|\bm \beta \|_1 + I_{\mathcal C}(\bm \beta)\right)- \left(\lambda \|\bm \beta^{k+1} \|_1+ I_{\mathcal C}(\bm \beta^{k+1}) \right)+ (\bm \beta - \bm \beta^{k+1})^\top [ \mu \bm A_1^\top(\boldsymbol{A}_1{\boldsymbol{\beta}}^{k+1}+ \sum_{m=1}^{M}{\boldsymbol{B}_m}\boldsymbol{\beta}_m^k  \notag \\
+ \bm B_{M+1} \bm b^k-\boldsymbol{c}-\boldsymbol{z}^k/\mu ) + \bm S_{\bm G}\bm( \beta^{k+1} - \bm \beta^k)   ] \ge 0
\end{align}
\end{small}
For the ${\boldsymbol r_m}$-subproblem ($m = 1,2,\dots,M$) of (\ref{th4op}), it follows from the convexity of $\mathcal{L}(\bm r_m)$ that
\begin{align}\label{ggg3}
\mathcal{L}(\bm r_m) - \mathcal{L}(\bm r_m^{k+1})  + (\bm r_m - \bm r_m^{k+1})^\top [ \mu \bm A_2^\top(\bm A_2 \bm r_1^{k+1}  + \bm B_1 \bm \beta_1^k   - \bm c - \bm z^k/\mu    )   ] \ge 0
\end{align}
By combining the subproblems (\ref{ggg1}) and (\ref{ggg3}) together and performing some algebraic manipulations, we have
\begin{equation}\label{ggg4}
\bm \theta_1({\boldsymbol{v}_1})- \bm \theta_1({\boldsymbol{v}_1}^{k+1})+({\boldsymbol{v}_1}-{\boldsymbol{v}_1}^{k+1})^\top \left[\mu \boldsymbol{A}^{\top}(\boldsymbol{A}{\boldsymbol{v}_1}^{k+1}+ {\boldsymbol{B}}\boldsymbol{v}_2^k-\boldsymbol{c}-\boldsymbol{z}^k/\mu ) \right]  + (\bm \beta - \bm \beta^{k+1})^\top \bm S_{\bm G}\bm( \beta^{k+1} - \bm \beta^k) \ge 0.
\end{equation}
The key distinction between the proof of \cite{BY,BY2} and the current situation lies in the linearized term $\frac{1}{2}\|\bm \beta - \bm \beta^k\|_{ \bm S_{\bm G}}^{2}$. By extracting this linearized term separately, we can readily obtain (\ref{ggg4}).
Similarly, 
for the $\bm{v}_2$-subproblem of (\ref{th4op}), by the convexity of  $\bm \theta_2({\boldsymbol{v}_2})$, we have
\begin{equation}\label{ggg5}
\bm \theta_2({\boldsymbol{v}_2}) -\bm \theta_2({\boldsymbol{v}_2^{k+1}}) +(\bm v_2 - \bm v_2^{k+1})^\top \left[ \mu {\boldsymbol{B}}^\top (\boldsymbol{A}{\boldsymbol{v}_1}^{k+1}+ {\boldsymbol{B}}\boldsymbol{v}_2^{k+1}-\boldsymbol{c}-\boldsymbol{z}^k/\mu)    \right] \ge 0.
\end{equation}
For the $\boldsymbol{z}$-subproblem of (\ref{th4op}), we have
\begin{equation}\label{ggg6}
(\boldsymbol{z}-\boldsymbol{z}^{k+1})^\top\left[(\boldsymbol{z}^{k+1}-\boldsymbol{z}^{k})/\mu+ (\boldsymbol{A}{\boldsymbol{v}_1}^{k+1}+ {\boldsymbol{B}} \boldsymbol{v}_2^{k+1} -\boldsymbol c)\right] \ge 0.
\end{equation}
Summing the above three inequalities (\ref{ggg4}), (\ref{ggg5}) and (\ref{ggg6})  together, we obtain
\begin{equation}\label{ggg7}
\begin{split}
\bm{\theta}(\boldsymbol{{v}_{1,2}})&- \bm{\theta}(\boldsymbol{{v}_{1,2}}^{k+1})+(\tilde{\boldsymbol{w}}-\tilde{\boldsymbol{w}}^{k+1})^\top F(\tilde{\boldsymbol{w}}^{k+1})+ ({\boldsymbol{v}_1}-{\boldsymbol{v}_1}^{k+1})^\top \mu \boldsymbol{A}^\top {\boldsymbol{B}}(\boldsymbol{v}_2^k-\boldsymbol{v}_2^{k+1}) \\
&+ (\bm \beta - \bm \beta^{k+1})^\top \bm S_{\bm G}\bm( \beta^{k+1} - \bm \beta^k)+{1 \over \mu}(\boldsymbol{z}-\boldsymbol{z}^{k+1})^\top(\boldsymbol{z}^{k+1}-\boldsymbol{z}^{k}) \ge 0,
\end{split}
\end{equation}
where  $F(\tilde{\bm w}^{k+1})=\left[ {\begin{array}{*{20}{c}}
-\bm A^\top \bm z^{k+1}\\
-{\bm B}^\top \bm z^{k+1}\\
\boldsymbol{A} {\bm{v}_1}^{k+1}+{\boldsymbol{B}}\boldsymbol{v}_2^{k+1}-\boldsymbol{c}
\end{array}} \right]$.

Add $(\boldsymbol{v}_2-\boldsymbol{v}_2^{k+1})^\top \mu {\boldsymbol{B}}^\top {\boldsymbol{B}}(\boldsymbol{v}_2^k-\boldsymbol{v}_2^{k+1}) $ to both sides of inequality (\ref{ggg7}), then 
\begin{equation}\label{ggg8}
\begin{split}
\bm \theta(\boldsymbol{ v}_{1,2})- \bm \theta(\boldsymbol{v}_{1,2}^{k+1})+(\tilde{\boldsymbol{w}}-\tilde{\boldsymbol{w}}^{k+1})^\top F(\boldsymbol{\tilde{w}}^{k+1})+& \mu {\left[ \begin{array}{*{20}{c}}
\boldsymbol{v}_1  - {\boldsymbol{v}_1 ^{k + 1}}\\
\boldsymbol{v}_2  - {\boldsymbol{v}_2 ^{k + 1}}
\end{array} \right]^\top}\left[ \begin{array}{l}
\boldsymbol{A}^\top\\
{\boldsymbol{B}}^\top
\end{array} \right] {\boldsymbol{B}}(\boldsymbol{v}_2^k-\boldsymbol{v}_2^{k+1})\\
&\ge  (\tilde{\boldsymbol{v}}-\tilde{\boldsymbol{v}}^{k+1})^\top \boldsymbol{H} (\tilde{\boldsymbol{v}}^k- \tilde{\boldsymbol{v}}^{k+1}),
\end{split}
\end{equation}
where $\tilde{\bm v} =  \left(\bm \beta^\top, \bm v_2^\top ,\boldsymbol{z}^\top \right)^\top$ and $\boldsymbol{H} = \begin{bmatrix}
\bm S_{\bm G}   & \bm 0 & \bm 0 \\
\bm 0 &\mu \bm B^\top \bm B & \bm 0 \\
\bm 0 &\bm 0 & \frac{1}{\mu} \bm I_{n_{\bm G} - 1 +(M+1)p}
\end{bmatrix} .$

From the variational inequality in section 2 of  \cite{BY}, we observe that
\begin{equation}\label{ggg9}
\bm \theta(\boldsymbol{ v}_{1,2})- \bm \theta(\boldsymbol{v}_{1,2}^{k+1})+(\tilde{\boldsymbol{w}}-\tilde{\boldsymbol{w}}^{k+1})^\top F(\tilde{\boldsymbol{w}}^{k+1})=\bm \theta(\boldsymbol{ v}_{1,2}^*)- \bm \theta(\boldsymbol{v}_{1,2}^{k+1})+(\tilde{\boldsymbol{w}}^*-\tilde{\boldsymbol{w}}^{k+1})^\top F(\tilde{\boldsymbol{w}}^{*}) \le  0.
\end{equation}
Together with $\bm A \bm v_1^* + \bm B \bm v_2^* = \bm c$,  and after performing simple algebraic operations,  we obtain
\begin{equation}\label{ggg10}
 \mu {\left[ \begin{array}{*{20}{c}}
\boldsymbol{v}_1^*  - {\boldsymbol{v}_1 ^{k + 1}}\\
\boldsymbol{v}_2^* - {\boldsymbol{v}_2 ^{k + 1}}
\end{array} \right]^\top}\left[ \begin{array}{l}
\boldsymbol{A}^\top\\
{\boldsymbol{B}}^\top
\end{array} \right] {\boldsymbol{B}}(\boldsymbol{v}_2^k-\boldsymbol{v}_2^{k+1})= (\boldsymbol{z}^{k+1}-\boldsymbol{z}^{k})^\top {\boldsymbol{B}} (\boldsymbol{v}_2^k-\boldsymbol{v}_2^{k+1}).
\end{equation}
Let $\bm v_1 = \bm v_1^*, \bm v_2 = \bm v_2^*, \bm v_{1,2}= \bm v_{1,2}^*, \tilde{\bm v}= \tilde{\bm v}^*, \tilde{\bm w} =\tilde{\bm w}^*$ and bring (\ref{ggg9}) and (\ref{ggg10}) into (\ref{ggg8}), then we get
\begin{equation}
(\tilde{\boldsymbol{v}}^{k+1}-\tilde{\boldsymbol{v}}^{*})^\top \boldsymbol{H} (\tilde{\boldsymbol{v}}^k-\tilde{\boldsymbol{v}}^{k+1}) \ge  (\boldsymbol{z}^{k+1}-\boldsymbol{z}^{k})^\top {\boldsymbol{B}} (\boldsymbol{v}_2^k-\boldsymbol{v}_2^{k+1}).
\end{equation}
From (\ref{ggg5}), we obtain the following two inequalities
\begin{align}
\bm \theta_2({\boldsymbol{v}_2}^k) -\bm \theta_2({\boldsymbol{v}_2^{k+1}}) 
+(\boldsymbol{v}_2^k-\boldsymbol{v}_2^{k+1})^\top (-\bm B^\top \boldsymbol{z}^{k+1}) \ge 0,\\
\bm \theta_2({\boldsymbol{v}_2}^{k+1}) -\bm \theta_2({\boldsymbol{v}_2^{k}}) 
+(\boldsymbol{v}_2^{k+1}-\boldsymbol{v}_2^{k})^\top (-\bm B^\top \boldsymbol{z}^{k}) \ge 0.
\end{align}
These two inequalities point out that
\begin{equation}
 (\boldsymbol{z}^{k}-\boldsymbol{z}^{k+1})^\top \bm B (\boldsymbol{v}_2^k-\boldsymbol{v}_2^{k+1}) \ge 0.
\end{equation}
Then, we obtain
\begin{equation}\label{pro6}
(\tilde{\boldsymbol{v}}^{k+1}-\tilde{\boldsymbol{v}}^{*})^\top \boldsymbol{H} (\tilde{\boldsymbol{v}}^k-\tilde{\boldsymbol{v}}^{k+1})  \ge 0.
\end{equation}
It is clear that  $\boldsymbol{H}$ is a symmetric and positive definite   matrix. Then, for any $\boldsymbol{x}, \boldsymbol {y}$, if $\boldsymbol{y}^\top \boldsymbol{H}(\boldsymbol {x}-\boldsymbol {y}) \ge 0$, then we have 
\begin{equation}\label{key}
\|\boldsymbol{y} \|_{\boldsymbol{H}}^2 \le \|\boldsymbol{x} \|_{\boldsymbol{H}}^2- \|\boldsymbol{x}-\boldsymbol{y} \|_{\boldsymbol{H}}^2.
\end{equation}
Therefore, 
$$\| \tilde{\boldsymbol{v}}^{k+1}-\tilde{\boldsymbol{v}}^* \|_{\boldsymbol{H}}^{2} \le \| \tilde{\boldsymbol{v}}^{k}-\tilde{\boldsymbol{v}}^* \|_{\boldsymbol{H}}^{2} - \| \tilde{\boldsymbol{v}}^{k}-\tilde{\boldsymbol{v}}^{k+1} \|_{\boldsymbol{H}}^{2}.$$ 
So far, we have completed the proof of Proposition \ref{prop8}. 
\end{proof}

The  contraction inequality (\ref{iee}) plays a crucial role in the convergence analysis of the ADMM algorithm. Proposition \ref{prop8} has a similar conclusion to Lemma 3 in \cite{LMY}. Based on the  contraction inequality (\ref{iee}), the assertions summarized in the following corollary become trivial.
\begin{cor}\label{cor1}
Let ${\tilde{\bm v}^k}$ be the sequence generated by Algorithm \ref{alg1}. Then we have
\vspace{-1em}
\begin{enumerate}
\item $\mathop {\lim }\limits_{k \to \infty}\|\tilde{\bm v}^k-\tilde{\bm v}^{k+1}\|_{\bm H}=0$;
\item the sequence  $\{\tilde{\bm v}^k\}$ is bounded;
\item for any optimal solution $\tilde{\bm v}^*$, the sequence $\{\|\tilde{\bm v}^k- \tilde{\bm v}^* \|_{\bm H} \}$ is monotonically non-increasing.
\end{enumerate}
\end{cor}

By utilizing Corollary \ref{cor1}, similar to the proof of Theorem 1 in \cite{LMY}, we can establish the convergence of $\bm{v}^k$ towards $\bm{v}^*$. Additionally, the convergence rate of $O(1/k)$ in a non-ergodic sense can be directly derived from Theorem 5.1 in \cite{BY2}. Consequently, we have successfully demonstrated the validity of Theorem \ref{TH2}.
%%%%%%%

\section{Supplementary experiments on nonconvex combined regularizations}\label{C}
In this subsection, we present the simulation results of CPADMM for solving nonconvex combined regularization regressions. Table \ref{Tab9} presents the performance of the non-convex extensions (mnet, snet) of the elastic net on the parallel dataset in Section \ref{sec611}. 
\begin{table}[H]\footnotesize
\centering
\caption{\footnotesize{The results of mnet and snet solved by CPADMM.}}
\resizebox{\textwidth}{!}{
\begin{threeparttable}
\begin{tabular}{lcccccccc}
\Xhline{1pt}
\ &Method & AAE($10^{-2}$) & ASE($10^{-2}$) & AAP & ASP & Time \\
\Xhline{0.4pt}
\multirow{2}*{M=1}   & mnet & 0.256(0.009) & 0.795(0.011) & \bf{0.136(0.005)} & \bf{0.166(0.006)} & 206.7(24.3)\\
                     & snet & \bf{0.241(0.008)} & \bf{0.787(0.010)} & 0.143(0.005) & 0.171(0.007) & 203.8(25.1) \\ 
\multirow{2}*{M=10}  & mnet & 0.289(0.010) & 0.813(0.012) & 0.152(0.006) & 0.184(0.007) & 37.9(4.3) \\ 
                     & snet & 0.290(0.011) & 0.824(0.011) & 0.148(0.006) & 0.181(0.005) & 36.5(4.1) \\ 
\multirow{2}*{M=100} & mnet & 0.321(0.012) & 0.923(0.018) & 0.159(0.007) & 0.199(0.008) & 15.6(1.8) \\ 
                     & snet & 0.337(0.013) & 0.941(0.019) & 0.163(0.006) & 0.203(0.008) & \bf{14.7(1.7)} \\ 
\Xhline{1pt}
\end{tabular}
\begin{tablenotes}
        \footnotesize{
        \item[*] The meanings of the notations used in this table are as follows: AAE: average absolute estimation error; ASE: average square estimation error; AAP: average absolute prediction error;  ASP: average square prediction error. Numbers in the parentheses represent the corresponding standard deviations. The optimal solution is represented in bold.}
\end{tablenotes}
\end{threeparttable}}
\label{Tab9}
\end{table}
Table \ref{Tab10} presents the performance of the non-convex extensions (mctv, sctv) of the sparse fused lasso on the parallel dataset in Section \ref{sec612}.
\begin{table}[H]\footnotesize
\centering
\caption{\footnotesize{The results of  mctv and sctv solved by CPADMM.}}
\resizebox{\textwidth}{!}{
\begin{threeparttable}
\begin{tabular}{lcccccccc}
\Xhline{1pt}
\ &Method & AAE($10^{-2}$) & ASE($10^{-2}$) & AAP & ASP & Time \\
\Xhline{0.4pt}
\multirow{2}*{M=1}   & mctv & \bf{0.006(0.001)} & 0.015(0.003) & \bf{0.432(0.011)} & 0.518(0.015) & 197.8(22.4)\\
                     & sctv & 0.007(0.001) & \bf{0.014(0.002)} & 0.451(0.010) & \bf{0.493(0.014)} & 201.1(21.5) \\ 
\multirow{2}*{M=10}  & mctv & 0.009(0.002) & 0.018(0.003) & 0.478(0.012) & 0.553(0.016) & 30.2(3.9) \\ 
                     & sctv & 0.009(0.002) & 0.017(0.003) & 0.487(0.012) & 0.523(0.015) & 32.7(4.1) \\ 
\multirow{2}*{M=100} & mctv & 0.011(0.003) & 0.021(0.003) & 0.538(0.017) & 0.682(0.021) & 12.3(1.5) \\ 
                     & sctv & 0.012(0.003) & 0.023(0.003) & 0.557(0.018) & 0.659(0.019) & \bf{11.9(1.6)} \\ 
\Xhline{1pt}
\end{tabular}
\begin{tablenotes}
        \footnotesize{
        \item[*] The meanings of the notations used in this table are as follows: AAE: average absolute estimation error; ASE: average square estimation error; AAP: average absolute prediction error;  ASP: average square prediction error. Numbers in the parentheses represent the corresponding standard deviations. The optimal solution is represented in bold.}
\end{tablenotes}
\end{threeparttable}}
\label{Tab10}
\end{table}

Table \ref{Tab11} presents the performance of the non-convex extensions (mcgl, scgl) of the sparse fused lasso on the parallel dataset in Section \ref{sec613}.
\begin{table}[H]\footnotesize
\centering
\caption{\footnotesize{The results of  mcgl and scgl solved by CPADMM.}}
\resizebox{\textwidth}{!}{
\begin{threeparttable}
\begin{tabular}{lcccccccc}
\Xhline{1pt}
\ &Method & AAE($10^{-2}$) & ASE($10^{-2}$) & AAP & ASP & Time \\
\Xhline{0.4pt}
\multirow{2}*{M=1}   & mcgl & \bf{0.012(0.002)} & 0.029(0.005) & \bf{0.236(0.008)} & \bf{0.259(0.010)} & 210.6(23.5)\\
                     & scgl & 0.015(0.003) & \bf{0.027(0.004)} & 0.286(0.007) & 0.276(0.011) & {201.3(22.9)} \\ 
\multirow{2}*{M=10}  & mcgl & 0.018(0.004) & 0.035(0.006) & 0.297(0.008) & 0.283(0.012) & 41.3(4.2) \\ 
                     & scgl & 0.021(0.005) & 0.037(0.005) & 0.310(0.009) & 0.291(0.012) & 38.4(3.8) \\ 
\multirow{2}*{M=100} & mcgl & 0.027(0.006) & 0.047(0.008) & 0.356(0.008) & 0.354(0.013) & 15.8(1.9)  \\ 
                     & scgl & 0.031(0.007) & 0.052(0.007) & 0.379(0.009) & 0.363(0.014) & \bf{14.3(1.8)} \\ 
\Xhline{1pt}
\end{tabular}
\begin{tablenotes}
        \footnotesize{
        \item[*] The meanings of the notations used in this table are as follows: AAE: average absolute estimation error; ASE: average square estimation error; AAP: average absolute prediction error;  ASP: average square prediction error. Numbers in the parentheses represent the corresponding standard deviations. The optimal solution is represented in bold.}
\end{tablenotes}
\end{threeparttable}}
\label{Tab11}
\end{table}

\end{document}